\documentclass{article}

\PassOptionsToPackage{numbers,compress}{natbib}
\usepackage[preprint]{neurips_2026}

\usepackage{graphicx} 
\usepackage[utf8]{inputenc} 
\usepackage{enumitem} 
\usepackage[T1]{fontenc}    
\usepackage{hyperref}       
\hypersetup{hidelinks}
\usepackage{url}            
\usepackage{booktabs}       
\usepackage{amsfonts}       
\usepackage{nicefrac}       
\usepackage{microtype}      
\usepackage{xcolor}         
\usepackage{comment}

\usepackage{algorithm}
\usepackage{algpseudocode}
\usepackage{amssymb} 
\usepackage{amsmath}
\usepackage{amsthm}
\newtheorem{theorem}{Theorem}
\newtheorem{proposition}{Proposition}
\newtheorem{corollary}{Corollary}

\newtheorem{lemma}{Lemma}
\usepackage{multirow}
\usepackage{subcaption} 
\usepackage{pifont}
\newcommand{\cmark}{\ding{51}}
\newcommand{\xmark}{\ding{55}}

\usepackage{longtable}
\usepackage{float}
\usepackage{wrapfig}
\usepackage{caption}
\title{SPACE: Sample-cloud Predictive Adaptive Conformal Ellipsoids for Multivariate Time-Series Forecasting}

\author{
Baishi Li$^{1}$ \qquad
Kelvin J.L. Koa$^{2}$ \qquad
Ke-Wei Huang$^{1,2}$ \\[4pt]
$^{1}$Department of Information Systems and Analytics,
National University of Singapore\\
$^{2}$Asian Institute of Digital Finance,
National University of Singapore\\[2pt]
\texttt{baishili@u.nus.edu, kelvin.koa@u.nus.edu, dishkw@nus.edu.sg}
}

\begin{document}

\maketitle

\begin{abstract}
Modern probabilistic time-series forecasters often express uncertainty through forecast samples. While typically converted into nominal prediction regions using empirical quantiles, these model-implied sets lack formal coverage guarantees and frequently deviate from nominal targets under distribution shift. Existing multivariate conformal methods can calibrate these regions online, but they typically estimate geometry from historical residuals using fixed or accumulating look-back windows. This reliance on the past limits their ability to exploit the instantaneous dependence structure of current predictions and leaves them vulnerable to stale-regime contamination. To address this, we propose \textsc{Space}, a conformal wrapper for sample-generating multivariate forecasters. \textsc{Space} constructs ellipsoidal joint prediction regions by estimating time-local covariance geometry directly from the current forecast sample cloud, calibrating the region's radius via a dynamic backward window-selection scheme. Across diverse multivariate datasets, probabilistic forecasters, and conformal baselines, \textsc{Space} consistently brings realized joint and rolling coverage closer to the nominal target, achieving superior coverage-efficiency tradeoffs relative to competing wrappers.

\end{abstract}

\section{Introduction}
\label{sec:intro}
While modern deep forecasting architectures---spanning diffusion \cite{rasul2021timegrad,li2024tmdm,ye2025nsdiff}, autoregressive \cite{salinas2020deepar,rangapuram2018deepstate}, copula \cite{tactis2,salinas2019gpcopula}, and flow-based models \cite{rasul2020cnf,kollovieh2024tsflow}---excel at generating multivariate sample clouds, converting these rich representations into statistically calibrated joint prediction regions under non-stationarity remains an open challenge. Building on foundational conformal prediction work \cite{vovk2005algorithmic,lei2018distribution,angelopoulos2023gentle}, recent methods have made major strides in sequential coverage calibration \cite{xu2021enbpi,gibbs2021aci,angelopoulos2023conformal}, often by using historical residual covariance to shape multivariate prediction sets \cite{xu2024multidimspci,messoudi2022ellipsoidal}. This approach is natural and effective, but it can adapt slowly when the instantaneous uncertainty geometry changes abruptly. In this paper, we build on these foundations by showing that robustness to sudden shifts can be improved by separating two roles that are often coupled: geometry estimation and conformal radius calibration. Our method shapes joint regions using the forecaster's current predictive sample covariance, while calibrating the radius using an adaptively selected, same-regime historical window.

The practical consequences of delayed geometric adaptation become visible under non-stationary dynamics. As illustrated in Figure~\ref{fig:intro_problem}(a), raw sample-based regions from modern forecasters can be substantially miscalibrated at the joint level across diverse datasets, often deviating due to model misspecification and finite-sample error \cite{romano2019cqr,marx2023distributionmatching}. Furthermore, as shown in Figure~\ref{fig:intro_problem}(b), this realized coverage often drifts away from the nominal $90\%$ target over time. Many multivariate conformal baselines correct this coverage gap by relying on past residuals to estimate both region geometry and calibration thresholds. After abrupt regime changes, this coupling creates a difficult trade-off: expanding the radius may recover coverage but sacrifice efficiency, while maintaining a tight region with outdated geometry preserves efficiency at the cost of local coverage. Even online adaptive methods, such as Adaptive Conformal Inference (ACI) \cite{gibbs2021aci,tibshirani2019weighted}, adjust only the scalar miscoverage rate, leaving the geometric shape inherited from historical residuals unchanged. The closest prior work, MultiDimSPCI \cite{xu2024multidimspci}, constructs ellipsoidal regions from historical residual covariance within a fixed-length sequential window. We address this limitation by ensuring that both the region's geometry and its calibration window are chosen adaptively in response to the current forecast state.

\begin{figure*}[t]
    \centering
    \begin{minipage}[t]{0.35\linewidth}
        \centering
        \includegraphics[width=\linewidth]{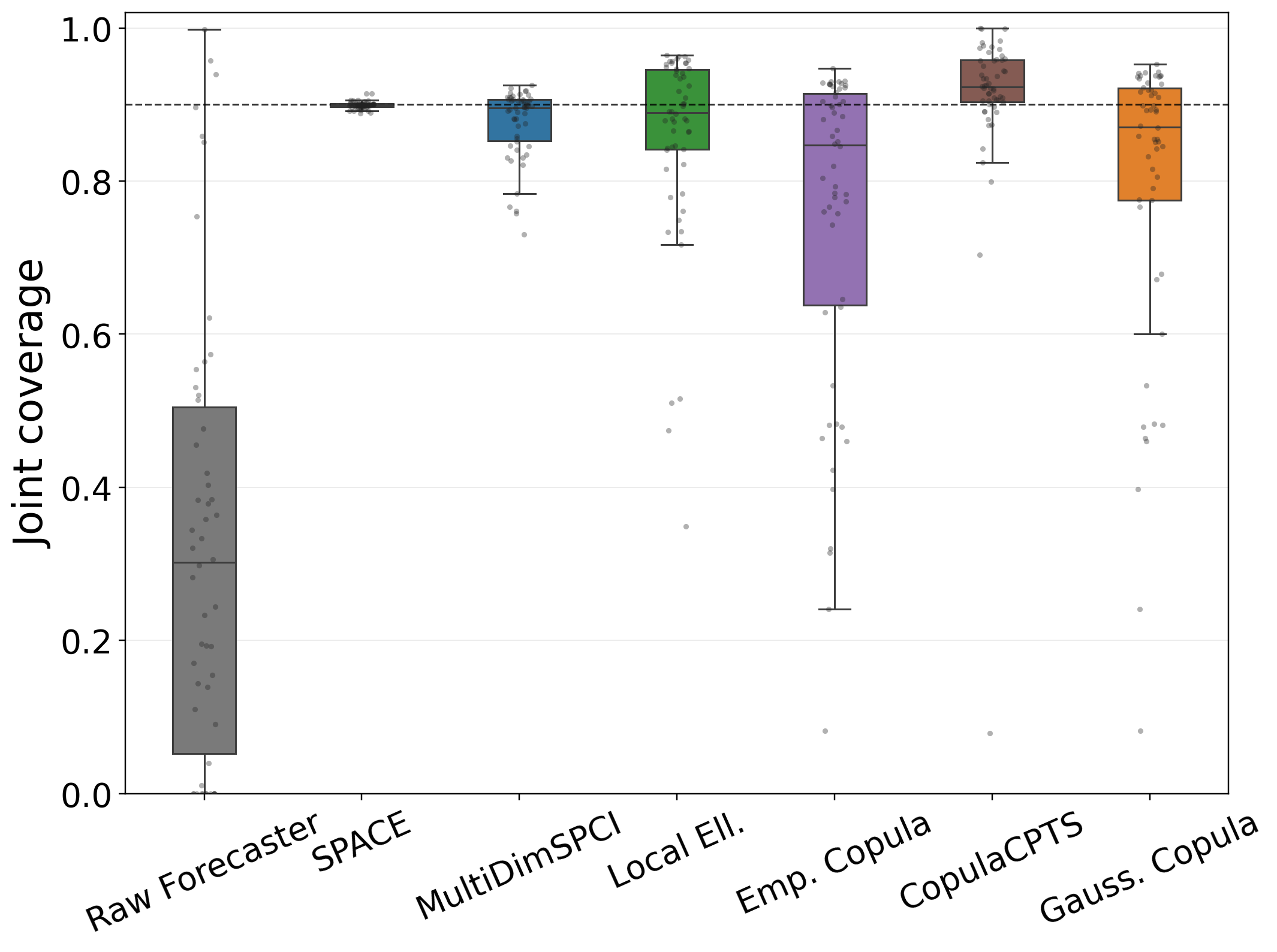}
        {\tiny (a)}
    \end{minipage}\hfill
    \begin{minipage}[t]{0.64\linewidth}
        \centering
        \includegraphics[width=\linewidth]{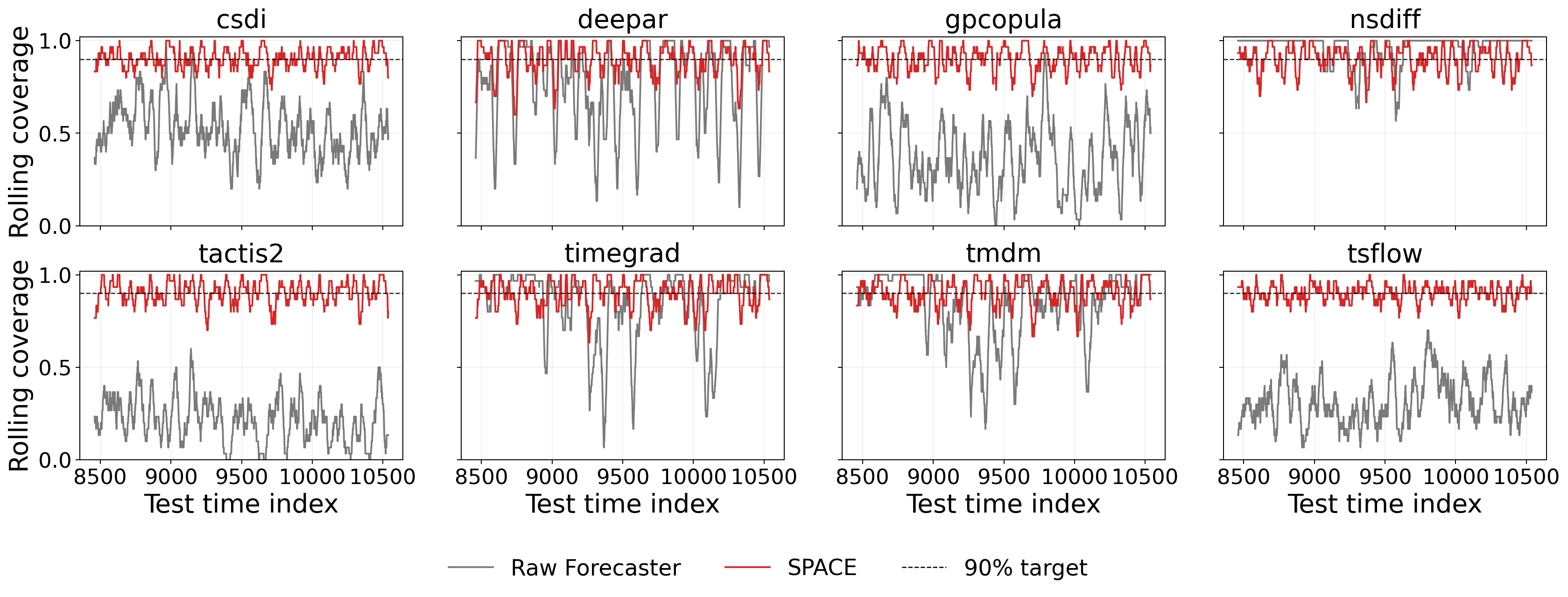}
        {\tiny (b)}
    \end{minipage}
    \caption{Joint coverage at the 90\% target. Panel (a) shows the distribution of empirical joint coverage. The x-axis evaluates different constructions: \texttt{Raw Forecaster} denotes native sample-based regions without conformal calibration, \textsc{Space} is our method, and the rest are competing conformal baselines. The y-axis shows empirical joint coverage, with the dotted line marking the 90\% target; better methods concentrate near this line. Panel (b) shows rolling joint coverage on the \texttt{weather} dataset across forecasters. Raw rolling coverage (gray) is unstable and systematically off-target, whereas \textsc{Space} (red) stays much closer to the desired level over time.}
    \label{fig:intro_problem}
\end{figure*}

To address these limitations, we introduce \textsc{Space} (Sample-cloud Predictive Adaptive Conformal Ellipsoids), a post-hoc conformal wrapper that directly exploits the probabilistic capabilities of the underlying forecaster. Our primary contributions are summarized as follows:

\begin{itemize}[topsep=0pt, itemsep=1pt]
    \item \textbf{Decoupling geometry from calibration.} We formalize a principle for multivariate conformal forecasting under non-stationarity: the forecaster's predictive samples should determine the local \textit{shape} of the joint region, while conformal calibration scores should determine its \textit{radius}. Based on this principle, we introduce \textsc{Space}, a post-hoc wrapper that constructs ellipsoidal regions using the time-local covariance of the current predictive sample cloud, reducing the adaptation delay associated with residual-based conformal geometries.
    \item \textbf{Regime-local adaptive calibration.} We propose a backward same-regime testing procedure for calibrating the radius of the forecast-conditioned ellipsoid. By adaptively selecting the longest recent calibration window whose conformity scores remain statistically compatible with the current regime, the method reduces the influence of outdated dynamics while balancing local coverage and statistical efficiency in sequential prediction.
    \item \textbf{Theoretical guarantees and error decomposition.} We establish finite-sample joint coverage under an idealized same-regime assumption and derive a sequential local coverage-gap bound that separates the effects of sample-based geometry estimation, adaptive window selection, and finite calibration sample size. This analysis clarifies how predictive sample size and calibration-window quality influence local coverage error.
    \item \textbf{Extensive empirical validation.} We evaluate \textsc{Space} across 7 multivariate datasets spanning energy, finance, weather, and traffic domains, 8 diverse probabilistic forecasters (including diffusion, flow, and autoregressive architectures), and 5 competing conformal baselines (MultiDimSPCI, local ellipsoidal CP, and copula-based CP). Across these diverse benchmark pairs, \textsc{Space} achieves a median realized joint coverage gap of just 0.2 percentage points from the nominal target. Furthermore, it reduces the overall mean coverage gap by over 80\% relative to the strongest baseline (MultiDimSPCI) while consistently improving local calibration stability under distribution shift.    
\end{itemize}

\section{Related Work}
\label{sec:related}

\paragraph{Sample-generating probabilistic forecasters.}
Modern forecasters—including autoregressive \cite{salinas2020deepar}, diffusion \cite{rasul2021timegrad,tashiro2021csdi,li2024tmdm,ye2025nsdiff}, copula \cite{salinas2019gpcopula,tactis2}, and flow-based models \cite{kollovieh2024tsflow}—encode joint dependencies through predictive sample clouds. However, these native sample regions remain model-implied rather than coverage-calibrated. We treat these models as black boxes, converting their raw predictive samples into formally calibrated joint prediction regions without altering the underlying architecture.

\paragraph{Sequential, localized, and online conformal prediction.}
A growing literature studies conformal prediction beyond the exchangeable setting. Methods such as EnbPI \cite{xu2021enbpi}, Adaptive Conformal Inference (ACI) \cite{gibbs2021aci}, and Conformal PID \cite{angelopoulos2023conformal} robustly adapt calibration thresholds or miscoverage levels online. Recent extensions have improved online tracking \cite{bhatnagar2023online}, rolling risk control \cite{angelopoulos2023conformal}, covariate shift and localized conformal prediction \cite{tibshirani2019weighted,barber2023conformal,guan2023localized}. However, these methods do not specify a forecast-conditioned geometry for \textbf{conformalized multivariate joint regions}.

\paragraph{Multivariate and ellipsoidal conformal prediction.} Prior work has developed conformal prediction for cross-sectional multivariate regression using structured dependence via copulas \cite{messoudi2021copula,sun2024copulacpts} and ellipsoidal sets \cite{messoudi2022ellipsoidal,johnstone2025exact,feldman2023conformal}. For multivariate time series, MultiDimSPCI \cite{xu2024multidimspci} is the closest benchmark to our work. It constructs sequential ellipsoidal regions by using historical residual covariance and handles temporal dependence through sequential calibration within a fixed-length window. \textsc{Space} differs by extracting the ellipsoidal geometry from the current predictive sample cloud---rather than from a historical surrogate---and by adaptively selecting the calibration window via same-regime testing. Another approach, FCP \cite{lee2026fcp}, targets multivariate time series by learning a context-conditioned residual flow, whereas \textsc{Space} requires no auxiliary generative model fitting.

\paragraph{Positioning of \textsc{Space}.}
\textsc{Space} is a conformal wrapper that simultaneously (i) operates on sample-generating multivariate forecasters without auxiliary model fitting, (ii) derives its multivariate geometry from the current predictive sample cloud rather than historical residuals, and (iii) adapts its calibration window to the current regime. Table~\ref{tab:positioning} summarizes these structural differences.

\begin{table}[t]
\centering
\caption{Positioning of \textsc{Space} relative to representative conformal approaches for multivariate and sequential prediction.}
\label{tab:positioning}
\small
\setlength{\tabcolsep}{4pt}
\resizebox{\linewidth}{!}{%
\begin{tabular}{lccccc}
\toprule
Method 
& Sequential TS 
& Joint multivariate set 
& Geometry source 
& Requires auxiliary CP model 
& Adaptive calibration window \\
\midrule
Adaptive TS CP \cite{xu2021enbpi,gibbs2021aci,angelopoulos2023conformal}
& \cmark & \xmark 
& Scalar / marginal residuals 
& \xmark 
& \xmark$^*$ \\
MultiDimSPCI \cite{xu2024multidimspci}
& \cmark & \cmark 
& Historical residual covariance 
& \xmark 
& \xmark \\
Copula-based CP \cite{messoudi2021copula,sun2024copulacpts}
& Mixed & \cmark 
& Fitted copula dependence 
& \cmark 
& \xmark \\
FCP \cite{lee2026fcp}
& \cmark & \cmark 
& Learned residual flow 
& \cmark 
& \xmark \\
\textsc{Space} (ours)
& \cmark & \cmark 
& Current predictive samples 
& \xmark 
& \cmark \\
\bottomrule
\multicolumn{6}{l}{\footnotesize $^*$These methods perform adaptive online scalar updates, but do not adaptively select contiguous historical calibration windows.}
\end{tabular}%
}
\end{table}

\section{Method: \textsc{Space}}
\label{sec:method}

We consider a probabilistic multivariate forecaster that outputs a set of $M$ predictive samples (e.g., generated trajectories from a conditional generative model) for a target $y_t \in \mathbb{R}^d$, where $d$ spans multiple variables, forecast horizons, or both. We denote this predictive sample cloud by
\begin{equation}
\widehat{\mathcal{Y}}_t = \{\hat y_t^{(m)}\}_{m=1}^M,
\qquad \hat y_t^{(m)} \in \mathbb{R}^d.
\label{eq:sample_cloud}
\end{equation}
Given a target miscoverage level $\alpha \in (0,1)$, our goal is to construct a joint prediction set
\begin{equation}
\mathcal{C}_t(\alpha) \subseteq \mathbb{R}^d
\quad \text{such that} \quad
\Pr\!\bigl(y_t \in \mathcal{C}_t(\alpha)\bigr) \approx 1-\alpha,
\label{eq:prediction_goal}
\end{equation}
while maximizing statistical efficiency.

Operating strictly causally over a chronological stream $\{(\widehat{\mathcal{Y}}_t, y_t)\}_{t=1}^T$, \textsc{Space} constructs this set sequentially. At prediction time $t$, it uses only the current sample cloud $\widehat{\mathcal{Y}}_t$ and realized past data ($s<t$). The method consists of two core components: First, \textsc{Space} extracts a time-local multivariate geometry directly from $\widehat{\mathcal{Y}}_t$ to shape the region (Section~\ref{subsec:local_geometry}). Second, it calibrates the ellipsoid's radius via a sequential same-regime procedure that dynamically selects the optimal historical calibration window (Section~\ref{subsec:dynamic_window}).

\subsection{Sample-derived local geometry}
\label{subsec:local_geometry}
For each time step $t$, we first construct a median-based center $\mu_t$ from $\widehat{\mathcal{Y}}_t=\{\hat y_t^{(m)}\}_{m=1}^M$, and form predictive residual samples
$r_t^{(m)} = \hat y_t^{(m)} - \mu_t$. To account for varying marginal scales, we compute the empirical covariance of the standardized residuals:
\begin{equation}
\widehat C_t
=
\operatorname{Cov}\!\bigl(\tilde r_t^{(1)},\dots,\tilde r_t^{(M)}\bigr),
\label{eq:samplecov}
\end{equation}
where $\tilde r_t^{(m)} = D^{-1} r_t^{(m)}$ and $D = \operatorname{diag}(s_1,\dots,s_d)$ with $s_j>0$ being the coordinate-wise scale estimate fitted from the wrapper-training segment. Mapping this covariance back to the original scale gives the local multivariate geometry at time $t$:
\begin{equation}
    \widehat\Sigma_t = D \widehat C_t D.
\label{eq:sigmahat}
\end{equation}
Crucially, as we establish in Proposition~\ref{prop:sample_cov_concentration}, this sample-derived covariance acts as a statistically consistent proxy that robustly concentrates around the forecaster's true conditional geometry.

We define the radial Mahalanobis-type nonconformity score at time $t$ as
\begin{equation}
e_t
=
\sqrt{(y_t-\mu_t)^\top \widehat\Sigma_t^{+}(y_t-\mu_t)}.
\label{eq:radial_score}
\end{equation}
Here, $\widehat\Sigma_t^{+}$ denotes the stabilized pseudo-inverse of the local covariance matrix $\widehat\Sigma_t$. Since $\widehat\Sigma_t$ is symmetric positive semidefinite, it admits a spectral decomposition $\widehat\Sigma_t = U_t \Lambda_t U_t^\top$, where $\Lambda_t=\operatorname{diag}(\lambda_{t,1},\dots,\lambda_{t,d})$ contains the nonnegative eigenvalues. The stabilized pseudo-inverse is formed by inverting only eigenvalues above an adaptive tolerance $\rho_t$:
\begin{equation}
\widehat\Sigma_t^{+}
=
U_t \operatorname{diag}\!\left(
\frac{\mathbf{1}\{\lambda_{t,i}>\rho_t\}}{\lambda_{t,i}}
\right) U_t^\top.
\label{eq:sigma_pinv}
\end{equation}
The score in \eqref{eq:radial_score} therefore measures how far the realized target lies from the center of the forecast cloud relative to its implied local covariance geometry.

\subsection{Dynamic same-regime calibration window selection}
\label{subsec:dynamic_window}
At forecast time $t$, let $\mathcal P_t=\{t-p,\dots,t-1\}$ denote the recent probe block of length $p$. \textsc{Space} searches backward over candidate calibration lengths
\begin{equation}
L \in \{L_{\min},\,L_{\min}+h,\,L_{\min}+2h,\dots,L_{\max}\},
\label{eq:candidate_lengths}
\end{equation}
where $h$ is the backward extension step size. For each candidate length $L$, define the full candidate block $\mathcal C_t^{\mathrm{whole}}(L)=\{t-p-L,\dots,t-p-1\}$ and the newly added increment $\mathcal B_t^{\mathrm{inc}}(L)=\{t-p-L,\dots,t-p-L+h-1\}$.

While the radial score $e_t$ captures overall ellipsoidal distance, it may miss regime shifts that manifest purely through coordinate-wise extremes or directional changes. Therefore, for the conformal uniformity test, we deploy a bundle of three complementary scalar diagnostics. For any observed time \(s\), let \(\tilde r_s=D^{-1}(y_s-\mu_s)\) denote the realized standardized residual:
\begin{equation}
s_{s,1}=e_s,\qquad
s_{s,2}=\|\tilde r_s\|_\infty,\qquad
s_{s,3}=|v_1^\top \tilde r_s|,
\label{eq:diagnostic_bundle}
\end{equation}
where $v_1$ is the top eigenvector of the training residual covariance in standardized space. Thus, $s_{s,1}$ captures radial size, $s_{s,2}$ captures extremeness, and $s_{s,3}$ captures variation along a dominant direction.

\paragraph{Conformal same-regime uniformity test.}
For each diagnostic $j \in \{1,2,3\}$ and candidate block $\mathcal A \in \{\mathcal C_t^{\mathrm{whole}}(L),\,\mathcal B_t^{\mathrm{inc}}(L)\}$, \textsc{Space} compares the probe scores $\{s_{u,j}:u\in\mathcal P_t\}$ with the calibration scores $\{s_{v,j}:v\in\mathcal A\}$ using randomized conformal p-values:
\begin{equation}
U_{u,j}(\mathcal A)=\frac{1+\#\{v \in \mathcal A : s_{v,j} > s_{u,j}\}+\xi_{u,j}\,\#\{v \in \mathcal A : s_{v,j} = s_{u,j}\}}{|\mathcal A| + 1},\qquad\xi_{u,j}\sim \mathrm{Unif}(0,1).
\label{eq:probe_pvalue}
\end{equation}
The finite-sample validity of this test is formalized in Assumption~(A2): clean same-regime blocks have controlled false-rejection probability, while contaminated extensions have controlled false-acceptance probability. Proposition~\ref{prop:window_recovery} shows that these local test guarantees imply recovery of the oracle same-regime window with high probability. Accordingly, $\mathcal A$ passes the test if, for each diagnostic $j$:
\begin{equation}
\mathrm{KS}_{t,L,j}^{\,\mathcal A}
=
\sup_{x \in [0,1]}
\left| \widehat F_{t,L,j}^{\,\mathcal A}(x) - x \right|
<
\tau_{t,L,j}^{\,\mathcal A},
\qquad j=1,2,3,
\label{eq:same_regime_test}
\end{equation}
where $\widehat F_{t,L,j}^{\,\mathcal A}(x)$ is the empirical CDF of the probe p-values, $\mathrm{KS}_{t,L,j}^{\,\mathcal A}$ is the Kolmogorov--Smirnov discrepancy, and $\tau_{t,L,j}^{\,\mathcal A}$ is the familywise-corrected Dvoretzky--Kiefer--Wolfowitz (DKW) threshold.

\paragraph{Sequential extension rule.}
Starting from $L_{\min}$, \textsc{Space} extends the calibration window from $L$ to $L+h$ only if both $\mathcal C_t^{\mathrm{whole}}(L)$ and $\mathcal B_t^{\mathrm{inc}}(L)$ pass \eqref{eq:same_regime_test}. This dual-check guards against dilution, since a stale older increment could be masked by the previously accepted history if only the full block were tested. The selected calibration length $L_t$ is the last accepted length before the first rejection. 

\subsection{Prediction Set}
\label{subsec:prediction_set}

Given $L_t$, \textsc{Space} calibrates the radius using the same-regime score history and the probe block. Let
\begin{equation}
\mathcal H_t(L_t)
=
\{e_s : s \in \mathcal C_t^{\mathrm{whole}}(L_t) \cup \mathcal P_t\}
\label{eq:selected_score_history}
\end{equation}
denote the final calibration score set available at time $t$. The radius is defined as
\begin{equation}
r_t
=
Q_{1-\alpha_t}\!\bigl(\mathcal H_t(L_t)\bigr),
\label{eq:radius_quantile}
\end{equation}
where $Q_{1-\alpha_t}(\cdot)$ is the empirical quantile at level $1-\alpha_t$, and $\alpha_t$ is updated online using adaptive conformal inference (ACI) \cite{gibbs2021aci}. The final prediction set at time $t$ is the ellipsoid
\begin{equation}
\mathcal C_t(\alpha)
=
\left\{
y \in \mathbb{R}^d :
(y-\mu_t)^\top \widehat\Sigma_t^{+}(y-\mu_t)
\le r_t^2
\right\}.
\label{eq:prediction_set}
\end{equation}

Figure~\ref{fig:geometry_weather_csdi_d11,12} visualizes the local prediction geometry used by our method. The ellipses are not fixed in either orientation or aspect ratio: instead, they rotate, elongate, or contract over time to match the local structure of the forecast sample cloud. This behavior is a direct consequence of our time-local covariance estimated from the contemporaneous generative sample cloud. 

\begin{figure}[t]
    \centering
    \includegraphics[width=0.5\linewidth]{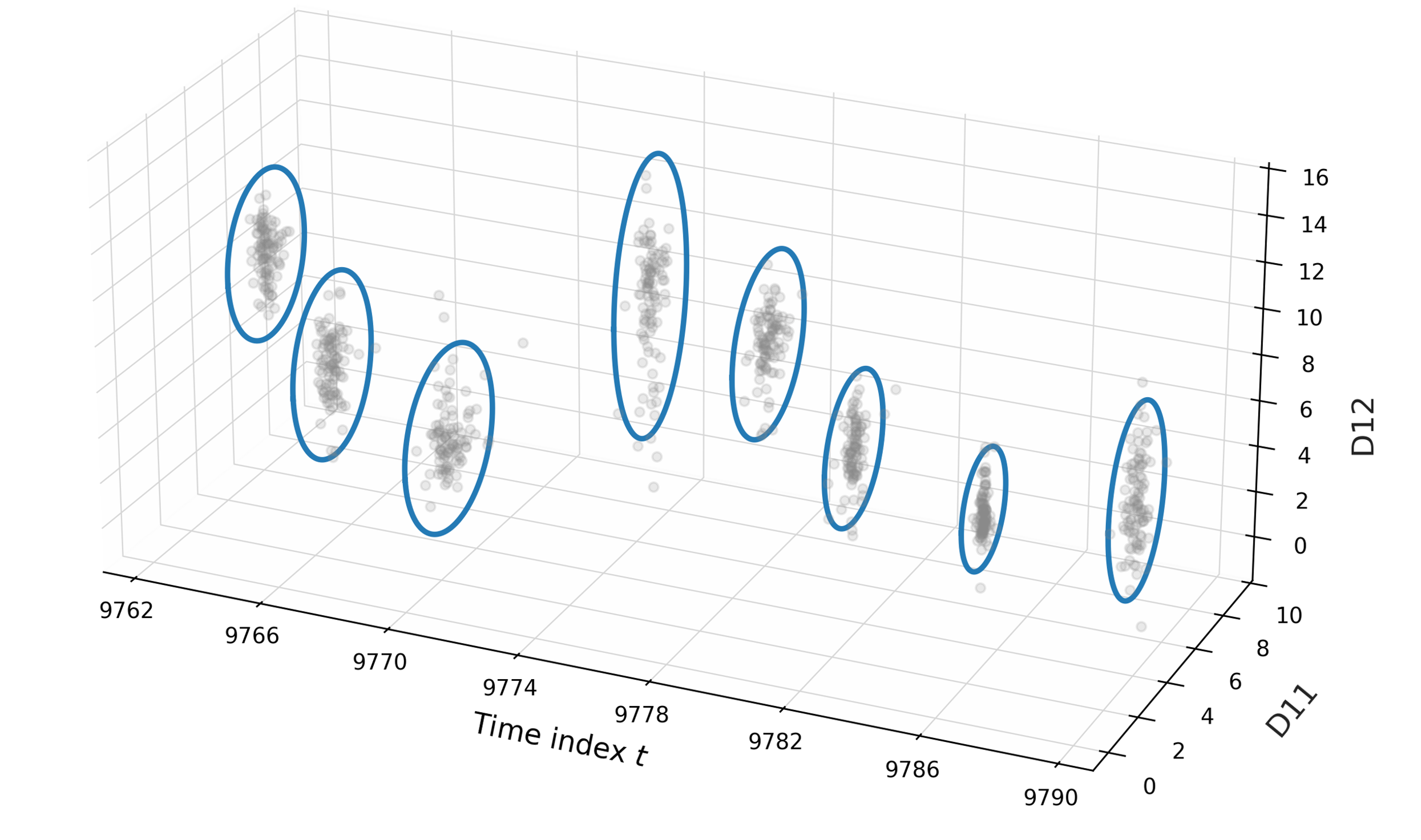}
    \caption{Illustration of local prediction geometry over time on the \texttt{Weather--CSDI} stream, shown on randomly selected target dimensions ($d_{11}$ and $d_{12}$) and time indices. Each ellipse is the 2D projection of \textsc{Space}'s ellipsoidal prediction region, while the gray points denote the forecast sample cloud.}
    \label{fig:geometry_weather_csdi_d11,12}
\end{figure}

By rigorously decoupling the sample-based shape estimation from the same-regime scale calibration, \textsc{Space} transcends empirical heuristics to achieve the formal finite-sample coverage guarantees and theoretical error decomposition presented next in Section~\ref{sec:theory}.

\paragraph{Computational Efficiency.}
Despite performing a dynamic backward search at every time step, \textsc{Space} introduces minimal operational overhead. The backward search relies strictly on one-dimensional empirical CDF comparisons (KS test) over the pre-computed scalar $s_{t,j}$, thereby bypassing repeated matrix inversions. Algorithm~\ref{alg:space} in Appendix \ref{app:appendix_algo} summarizes the full execution.

\section{Theoretical Analysis}
\label{sec:theory}

We analyze through: (i) finite-sample concentration of the sample-cloud covariance, (ii) recovery of a same-regime calibration window, and most importantly --- (iii) a sequential coverage-error decomposition. The four sets of formal assumptions and other proofs are deferred to Appendix~\ref{app:theory_proofs}.

\paragraph{Assumptions (informal).}
Let $\mathcal F_t$ be the filtration generated by realized targets, forecast clouds, and auxiliary randomness up to time $t$. We define $\mathbb P_t(\cdot) = \mathbb P(\cdot \mid \mathcal F_{t-1})$ as the probability conditional on the immediate past. For same-regime properties, we use $\mathbb Q_t(\cdot) = \mathbb P(\cdot \mid \mathcal F_{\tau_t})$, conditioning only on information prior to the most recent change point $\tau_t$. Our per-time coverage bounds are thus conditional; unconditional variants follow by integrating over the past.

\textbf{(A1) Sample-cloud regularity and stable score geometry:} conditional on the past, the $M$ forecast samples are i.i.d. sub-Gaussian draws from the forecaster-implied law, their standardized population covariance is well defined, and the resulting Mahalanobis conformal set is locally stable under small perturbations of the sample-cloud covariance \cite{wainwright2019high, xu2024multidimspci}. This assumption provides the regularity needed for the $L_{\Sigma}$ term below and rules out discontinuities caused by atoms in the ideal score distribution or eigenvalues crossing the pseudo-inverse threshold.\footnote{Assumption A1 is natural for generative architectures whose output samples are obtained by passing light-tailed latent variables through sufficiently regular maps. For example, if a standard Gaussian latent variable is transformed by an $L$-Lipschitz generator $f_\theta$, Gaussian concentration implies sub-Gaussian output behavior up to the Lipschitz scale \cite{vershynin2018high}. Spectral normalization is one common way to encourage such Lipschitz control in neural generators \cite{miyato2018spectral}.}
\textbf{(A2) Same-regime test calibration and shift detectability:} clean candidate blocks are rejected by the same-regime test with probability at most $\delta_{\mathrm{uni}}$, while every contaminated extension has at least one diagnostic/block comparison whose false-acceptance probability is at most $\beta_{\mathrm{det}}$ \cite{xie2012change, bhatnagar2023online}. For the KS-style version used by \textsc{Space}, a sufficient DKW-type form is $\beta_{\mathrm{det}}=2\exp(-2|\mathcal P_t|\Delta_{\mathrm{sep}}^2)$ under a population separation margin $\Delta_{\mathrm{sep}}>0$.
\textbf{(A3) Piecewise augmented exchangeability:} within each stationary regime, the augmented forecast--target objects $\mathcal Z_t=(Y_t,\widehat{\mathcal Y}_t)$ are conditionally exchangeable, and the ideal conformal scores used for coverage inherit this same-regime exchangeability \cite{barber2023conformal, gibbs2021aci}.
\textbf{(A4) Adaptive-level validity:} the adaptive miscoverage level $\alpha_t$ is causal and does not use the oracle test score. Formally, its effect is controlled by a tracking term $\varepsilon_t^{\mathrm{ACI}}$ together with a rank-validity condition for the oracle same-regime conformal predictor using $\alpha_t$, up to the usual finite calibration-sample error \cite{gibbs2021aci, angelopoulos2023conformal}.

\paragraph{Step 1: Geometry Concentration.}
First, we establish that the empirical sample-cloud covariance $\widehat C_t$ reliably estimates the forecaster-implied conditional geometry.
\begin{proposition}[Sample-cloud covariance concentration]
\label{prop:sample_cov_concentration}
Under (A1), for any fixed prediction time $t$ and any $\delta\in(0,1)$, there exists a universal constant $c_0>0$ such that, conditional on $\mathcal F_{t-1}$,
\begin{equation}
\mathbb P_t\!\left(
\left\|\widehat C_t-C_t^\star\right\|_{\mathrm{op}}
\le
\eta_{M,\delta}
\right)
\ge 1-\delta,
\qquad
\eta_{M,\delta}
=
c_0K^2\left(
\sqrt{\frac{d+\log(1/\delta)}{M}}
+
\frac{d+\log(1/\delta)}{M}
\right).
\label{eq:sample_cov_concentration}
\end{equation}
\end{proposition}
\textbf{Proof intuition:} \textit{This follows from the standard
operator-norm concentration bound for sample covariance matrices of
sub-Gaussian vectors (see \cite{vershynin2018high}).}

\paragraph{Remark on Statistical Rate and Generative Capacity.}
The convergence rate $\eta_{M,\delta}$ matches the standard dimension dependence for unstructured sub-Gaussian covariance estimation. Unlike classical settings constrained by historical data, the generative sample size $M$ is fully controllable and can theoretically be arbitrarily large, driving the geometry estimation error to zero. Practically, $M$ need not be large: once $M \gtrsim d$, the leading $\sqrt{d/M}$ term decays rapidly. Faster estimation rates would require structural assumptions on $C_t^\star$, such as sparsity or low rank.

\paragraph{Step 2: Window Recovery.}
$\rho_t^{\mathrm{win}}=\mathbb Q_t(\widehat L_t\neq L_t^\star)$ denotes the window-selection error probability, where \(L_t^\star\) is the largest candidate length whose calibration and probe blocks lie in the current regime.

\begin{proposition}[Same-regime window recovery]
\label{prop:window_recovery}
Under (A2), if the same-regime tests have false-rejection probability at most \(\delta_{\mathrm{uni}}\) and false-acceptance probability at most \(\beta_{\mathrm{det}}\), then
\begin{equation}
\rho_t^{\mathrm{win}} \le 6|\mathcal G|\delta_{\mathrm{uni}} + |\mathcal G|\beta_{\mathrm{det}}.
\end{equation}
In the DKW-style separation case, $\beta_{\mathrm{det}} = 2\exp\{-2|\mathcal P_t|\Delta_{\mathrm{sep}}^2\}.$\footnote{The Dvoretzky--Kiefer--Wolfowitz inequality, refined by Massart \cite{dvoretzky1956asymptotic,massart1990tight}, guarantees that the chance of missing a regime shift drops exponentially fast as we collect more recent data ($|\mathcal P_t|$) or as the shift becomes more severe ($\Delta_{\mathrm{sep}}$).} 
\end{proposition}

\paragraph{Proof intuition.} Because the window extends sequentially, a false rejection at any clean length prematurely halts the search before reaching the oracle length $L_t^\star$. A union bound over the $6|\mathcal G|$ clean tests (2 block types $\times$ 3 diagnostics) controls this premature stopping. Conversely, Assumption (A2) ensures that any invalid length $L>L_t^\star$ is detected and rejected with probability at least $1-\beta_{\mathrm{det}}$. A combined union bound over these candidate lengths yields the final recovery probability \cite{xie2012change,bhatnagar2023online}.

\paragraph{Step 3: Coverage Gap Decomposition.}
Let $\mathcal H_t^\star = \mathcal C_t^{\mathrm{whole}}(L_t^\star)\cup \mathcal P_t, n_t^\star=|\mathcal H_t^\star|,$ and \(\mathcal C_t^{\mathrm{ACI}}(\alpha)\) denote the final prediction set when the nominal target level is \(\alpha\), allowing
the algorithm to use the adaptive internal level \(\alpha_t\). The per-time coverage error can be decomposed into window-selection error, adaptive-level error, finite calibration-sample error, and sample-geometry perturbation error.

\begin{theorem}
\label{thm:per_time_coverage}
Assume (A1)--(A4). For any fixed prediction time \(t\) and any
\(\delta_t\in(0,1)\), suppose
\(\eta_{M,\delta_t/(n_t^\star+1)}\le \eta_0\). Then
\begin{equation}
\left|
\mathbb Q_t\!\left(
Y_t\in\mathcal C_t^{\mathrm{ACI}}(\alpha)
\right)
-
(1-\alpha)
\right|
\le
\rho_t^{\mathrm{win}}
+
\varepsilon_t^{\mathrm{ACI}}
+
\frac{1}{n_t^\star+1}
+
L_\Sigma \eta_{M,\delta_t/(n_t^\star+1)}
+
\delta_t ,
\label{eq:per_time_gap_bound}
\end{equation}
where \(L_{\Sigma}\) is the score-geometry stability constant from (A1), and $\eta_{M,\delta}$ is the covariance concentration rate from Proposition~\ref{prop:sample_cov_concentration}. The window-selection term $\rho_t^{\mathrm{win}}$ satisfies Proposition~\ref{prop:window_recovery}.
\end{theorem}

\paragraph{Proof intuition.} 
The proof compares the empirical prediction set with an oracle same-regime set. The coverage gap decomposes into additive penalties for window-selection failure, ACI tracking, finite calibration size, covariance estimation, and the concentration failure probability. The term $\rho_t^{\mathrm{win}}$ accounts for selecting the wrong window without conditioning on data-dependent search paths \cite{vovk2005algorithmic, lei2018distribution}. Corollary~\ref{cor:overall_gap} follows by averaging over \(t\) and applying the tower property.

\begin{corollary}[Expected empirical average coverage]
\label{cor:overall_gap}
Under the conditions of Theorem~\ref{thm:per_time_coverage}, assume the per-time right-hand-side terms are deterministic or hold almost surely under the outer expectation. Let $\widehat{\mathrm{Cov}}_T = T^{-1}\sum_{t=1}^T \mathbf 1\{Y_t\in\mathcal C_t^{\mathrm{ACI}}(\alpha)\}$, and define $\bar\rho_T^{\mathrm{win}} = T^{-1}\sum_{t=1}^T\rho_t^{\mathrm{win}}$, $\bar\varepsilon_T^{\mathrm{ACI}} = T^{-1}\sum_{t=1}^T\varepsilon_t^{\mathrm{ACI}}$, $\bar\eta_T = T^{-1}\sum_{t=1}^T \eta_{M,\delta_t/(n_t^\star+1)}$, and $\bar\delta_T = T^{-1}\sum_{t=1}^T\delta_t$. Then
\begin{equation}
\left| \mathbb E[\widehat{\mathrm{Cov}}_T] - (1-\alpha) \right| \le \bar\rho_T^{\mathrm{win}} + \bar\varepsilon_T^{\mathrm{ACI}} + \frac1T\sum_{t=1}^T\frac{1}{n_t^\star+1} + L_{\Sigma}\bar\eta_T + \bar\delta_T .
\label{eq:overall_gap_bound}
\end{equation}
If these terms are random, interpret the averaged right-hand-side terms as their outer expectations.
\end{corollary}

\section{Experiments}
\label{sec:experiments}

\subsection{Experimental setup and metrics}
\label{subsec:exp_setup}

We evaluate \textsc{Space} on multivariate forecasting streams produced by eight heterogeneous probabilistic forecasters (TimeGrad \cite{rasul2021timegrad}, CSDI \cite{tashiro2021csdi}, TSFlow \cite{kollovieh2024tsflow}, DeepAR \cite{salinas2020deepar}, GP-Copula \cite{salinas2019gpcopula}, NsDiff \cite{ye2025nsdiff}, TMDM \cite{li2024tmdm}, and TACTiS-2 \cite{tactis2}) across seven real-world datasets summarized in Table~\ref{tab:dataset_summary}. We then compare the performance of \textsc{Space} against five conformal baselines (MultiDimSPCI \cite{xu2024multidimspci}, Local Ellipsoid \cite{messoudi2022ellipsoidal}, CopulaCPTS \cite{sun2024copulacpts}, Empirical Copula and Gaussian Copula \cite{messoudi2021copula}). The probabilistic forecasters are trained with a 7:1:2 train/validation/test split on the original dataset. Each conformal wrapper is evaluated with a 6:2:2 train/calibration/test split on the sample streams. 


In downstream decision settings, predictive uncertainty should be close to the nominal target, rather than merely exceed it \cite{zhao2021calibrating,kuleshov2022calibrated,deng2025uncertainty}. We therefore report both the coverage gap and the rolling coverage gap. The coverage gap is the absolute deviation between empirical joint coverage and the target level $1-\alpha$, and therefore penalizes both under- and over-coverage. To evaluate efficiency, we report the mean log-volume of the prediction regions, denoted \textsc{LogVol}, following prior multivariate conformal work that measures sharpness through region size or volume \cite{xu2024multidimspci,lee2026fcp,yang2026cp4gen}. This section focuses on the target coverage level $1-\alpha=0.90$. The full results at different target coverage levels and other implementation details are deferred to Appendix~\ref{app:appendix_result_table} and~\ref{app:space_impl}.

\subsection{Aggregate comparison against conformal wrappers}
\label{subsec:wrapper_results}

We first compare \textsc{Space} with five strong conformal wrappers from prior work. Table~\ref{tab:wrapper_summary} provides the wrapper-level comparison. The full dataset-forecaster level comparison is in Appendix~\ref{app:appendix_result_table}.

\begin{table}[t]
\captionsetup{aboveskip=\baselineskip}
\centering

\caption{Aggregated comparison over all dataset--forecaster pairs. Mean Gap and Median Gap summarize the pair-level absolute coverage gap from the target. Std. LogVol denotes the pair-level log-volume after standardization within each dataset and target level, so it measures size on a comparable scale across datasets. For Mean Gap and Mean Std. LogVol, the reported standard deviation is the across-pair standard deviation within each wrapper. Mean Std. LogVol and Median Std. LogVol are computed only over pairs whose coverage gap is at most $0.05$. Valid Ratio reports the fraction of dataset--forecaster pairs satisfying this coverage-gap threshold.}
\label{tab:wrapper_summary}

\small
\setlength{\tabcolsep}{2.5pt}
\begin{tabular}{lccccc}
\toprule
Wrapper & Mean Gap $\downarrow$ & Median Gap $\downarrow$ & Mean Std. LogVol $\downarrow$ & Median Std. LogVol $\downarrow$ & Valid Ratio \\
\midrule
SPACE            & \textbf{0.003 {\scriptsize $\pm$ 0.004}} & \textbf{0.002} & -0.111 {\scriptsize $\pm$ 1.043} & -0.334 & \textbf{1.000} \\
MultiDimSPCI     & 0.032 {\scriptsize $\pm$ 0.041} & 0.013 & -0.435 {\scriptsize $\pm$ 1.024} & \textbf{-0.804} & 0.759 \\
Local Ellipsoid  & 0.083 {\scriptsize $\pm$ 0.111} & 0.054 & \textbf{-0.580 {\scriptsize $\pm$ 0.878}} & -0.753 & 0.463 \\
Empirical Copula & 0.165 {\scriptsize $\pm$ 0.205} & 0.053 & 0.031 {\scriptsize $\pm$ 0.630} & 0.012 & 0.481 \\
Gaussian Copula  & 0.130 {\scriptsize $\pm$ 0.184} & 0.042 & 0.003 {\scriptsize $\pm$ 0.618} & -0.037 & 0.593 \\
CopulaCPTS       & 0.057 {\scriptsize $\pm$ 0.112} & 0.034 & 1.021 {\scriptsize $\pm$ 0.789} & 1.037 & 0.593 \\
\bottomrule
\end{tabular}

\end{table}

\textsc{Space} achieves the best Mean Gap and Median Gap among all compared wrappers, with the smallest standard deviation. This is the main wrapper-level result: the mean coverage gap of \textsc{Space} is $0.3$ percentage points from the nominal target. The efficiency results are more nuanced. \textsc{Space} does not attain the smallest LogVol overall; MultiDimSPCI has the lowest median volume. However, this smaller volume comes with noticeably worse calibration and a lower validity ratio. In particular, the mean coverage gap of \textsc{Space} is only about one tenth of that of the strongest baseline, MultiDimSPCI, and \textsc{Space} is the only method with a validity ratio of $1.00$. This indicates that \textsc{Space} is the most consistent wrapper for bringing realized coverage close to the nominal target across diverse datasets and forecasters.

The key point is therefore coverage--efficiency trade-off at the target level. \textsc{Space} achieves much better calibration while remaining competitive in median log-volume, which is the more relevant comparison for a conformal wrapper whose purpose is to deliver regions that are both reliable and usable.

\subsection{Rolling coverage under drift and regime change}
\label{subsec:rolling_results}

Rolling metrics show whether a wrapper stays close to the nominal target locally over time. This is a pattern that could be overlooked by aggregated mean or median measures. We use a rolling window length $w=30$ in all evaluations. Figure~\ref{fig:rolling_boxplots} reports boxplots for Mean Rolling Gap, P90 Rolling Gap, and Frac.\ Bad Windows, where lower values indicate better local calibration stability. Mean Rolling Gap averages the absolute deviation between rolling joint coverage and the target level $1-\alpha$ over all valid sliding windows. P90 Rolling Gap reports the $90$th percentile of these deviations, capturing severe local miscalibration. Frac.\ Bad Windows is the proportion of windows whose rolling coverage gap exceeds $\delta=0.10$. Table~\ref{tab:rolling_summary} complements the figure by providing the dataset-level breakdown.

 \begin{figure*}[t]
    \centering
    \includegraphics[width=\linewidth]{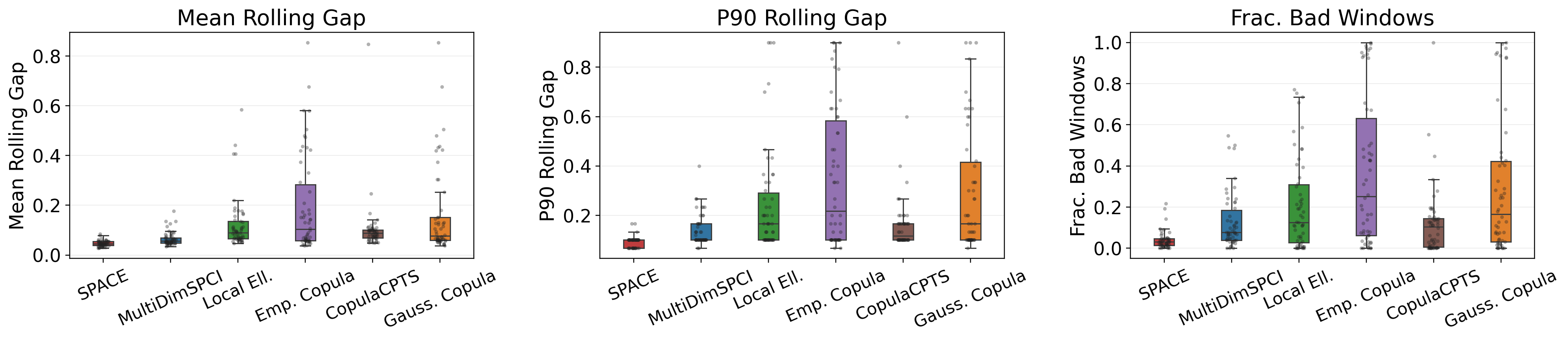}
    \caption{Distributional boxplots of rolling coverage diagnostics across wrapper benchmark pairs. The center line denotes the median. The three panels report Mean Rolling Gap, P90 Rolling Gap, and Frac.\ Bad Windows, respectively, with lower values indicating better local calibration stability. Table~\ref{tab:rolling_summary} gives the corresponding dataset-level breakdown.}
    \label{fig:rolling_boxplots}
\end{figure*}

\begin{table}[t]
\captionsetup{aboveskip=\baselineskip, belowskip=\baselineskip}
\centering

\caption{Dataset-level rolling coverage diagnostics for conformal wrappers at the $90\%$ target. Lower is better for all metrics. These metrics evaluate local calibration stability under drift and regime change. \textsc{Space} performs best or near-best on all datasets.}
\label{tab:rolling_summary}
\tiny
\begin{tabular}{llcccccc}
\toprule
Dataset & Metric & SPACE & MultiDimSPCI & Local Ell. & Emp. Copula & Gauss. Copula & CopulaCPTS\\
\midrule
\multirow{3}{*}{electricity} & Mean Rolling Gap $\downarrow$ & \textbf{0.041} & 0.109 & 0.278 & 0.326 & 0.125 & 0.082 \\
 & P90 Rolling Gap $\downarrow$ & \textbf{0.086} & 0.214 & 0.557 & 0.556 & 0.257 & 0.100 \\
 & Frac. Bad Windows $\downarrow$ & 0.021 & 0.360 & 0.584 & 0.689 & 0.298 & \textbf{0.010} \\
\midrule
\multirow{3}{*}{etth1} & Mean Rolling Gap $\downarrow$ & \textbf{0.032} & 0.044 & 0.074 & 0.048 & 0.049 & 0.072 \\
 & P90 Rolling Gap $\downarrow$ & \textbf{0.067} & 0.092 & 0.133 & 0.096 & 0.096 & 0.112 \\
 & Frac. Bad Windows $\downarrow$ & \textbf{0.006} & 0.014 & 0.074 & 0.020 & 0.010 & 0.041 \\
\midrule
\multirow{3}{*}{ettm1} & Mean Rolling Gap $\downarrow$ & 0.054 & \textbf{0.051} & 0.083 & 0.062 & 0.063 & 0.107 \\
 & P90 Rolling Gap $\downarrow$ & 0.108 & \textbf{0.100} & 0.146 & 0.112 & 0.108 & 0.204 \\
 & Frac. Bad Windows $\downarrow$ & 0.065 & 0.047 & 0.067 & 0.054 & \textbf{0.038} & 0.185 \\
\midrule
\multirow{3}{*}{ettm2} & Mean Rolling Gap $\downarrow$ & \textbf{0.045} & 0.055 & 0.070 & 0.067 & 0.066 & 0.065 \\
 & P90 Rolling Gap $\downarrow$ & \textbf{0.096} & 0.100 & 0.146 & 0.112 & 0.108 & 0.117 \\
 & Frac. Bad Windows $\downarrow$ & \textbf{0.029} & 0.061 & 0.118 & 0.091 & 0.072 & 0.077 \\
\midrule
\multirow{3}{*}{exchange} & Mean Rolling Gap $\downarrow$ & \textbf{0.050} & 0.067 & 0.156 & 0.257 & 0.250 & 0.211 \\
 & P90 Rolling Gap $\downarrow$ & \textbf{0.100} & 0.163 & 0.263 & 0.444 & 0.436 & 0.313 \\
 & Frac. Bad Windows $\downarrow$ & \textbf{0.066} & 0.162 & 0.225 & 0.460 & 0.435 & 0.296 \\
\midrule
\multirow{3}{*}{traffic} & Mean Rolling Gap $\downarrow$ & \textbf{0.055} & 0.069 & 0.103 & 0.467 & 0.467 & 0.072 \\
 & P90 Rolling Gap $\downarrow$ & \textbf{0.100} & 0.157 & 0.205 & 0.676 & 0.676 & 0.105 \\
 & Frac. Bad Windows $\downarrow$ & 0.039 & 0.174 & 0.281 & 0.950 & 0.950 & \textbf{0.038} \\
\midrule
\multirow{3}{*}{weather} & Mean Rolling Gap $\downarrow$ & \textbf{0.052} & 0.073 & 0.121 & 0.198 & 0.167 & 0.111 \\
 & P90 Rolling Gap $\downarrow$ & \textbf{0.096} & 0.142 & 0.288 & 0.488 & 0.408 & 0.204 \\
 & Frac. Bad Windows $\downarrow$ & \textbf{0.051} & 0.136 & 0.173 & 0.441 & 0.344 & 0.166 \\
\bottomrule
\end{tabular}

\end{table}

Across datasets, \textsc{Space} shows its clear advantage over the best baseline in mean rolling gap, especially on the higher-dimensional, more nonstationary settings, where maintaining stable coverage over time is most difficult. In particular, the separation between \textsc{Space} and the best competing method is largest on datasets such as \texttt{electricity} and \texttt{traffic}, which have richer joint dependence structure and stronger temporal instability; this suggests that \textsc{Space} is especially effective when the calibration problem is genuinely multivariate and the underlying coverage behavior drifts over time. By contrast, on smaller and more stationary datasets, the gap between methods narrows substantially. In \texttt{ettm1}, \textsc{Space} and the best baseline perform similarly. A more detailed analysis in Appendix~\ref{app:covdrift_performance} shows that \textsc{Space}’s gains are more pronounced in datasets with stronger covariance drift and more challenging rolling coverage behavior.

\section{Conclusion}

In this paper, we introduced \textsc{Space}, a post-hoc conformal wrapper that bridges the gap between the rich uncertainty representations of modern probabilistic forecasters and the rigorous calibration guarantees of conformal prediction. By decoupling the estimation of multivariate geometry---extracted directly from the predictive sample cloud---from the sequential calibration of the region's radius, \textsc{Space} overcomes the adaptation delays inherent in historical residual-based methods. Our theoretical analysis establishes formal finite-sample coverage guarantees under stated assumptions, and extensive experiments demonstrate that \textsc{Space} consistently tightens realized coverage toward the nominal target across diverse forecasting architectures (e.g., diffusion, flow, and autoregressive models) under non-stationarity. Ultimately, this work demonstrates that predictive sample clouds contain valuable geometric signals that can be exploited for calibrated predictive inference. Natural extensions of this framework include generalizing the set geometry beyond ellipsoids to capture non-convex predictive topologies.

\section{Limitations}
\label{sec:limitations}

This work has two main limitations. First, \textsc{Space} depends on the quality of the predictive sample cloud produced by the generative forecasters. The conformal calibration can correct the overall scale of the prediction region, but the local covariance geometry is still estimated from generated samples. Therefore, if the forecaster produces biased or insufficiently diverse samples, the calibrated regions may remain valid only at the cost of reduced efficiency. The second limitation concerns adaptation under difficult regime dynamics. When changes are gradual, recent scores may not show a sharp enough diagnostic signal; when changes are frequent, little same-regime history may be available. In either case, \textsc{Space} may select a window that is not fully representative of the current regime, leading to coverage deviations. This challenge is shared by time-series uncertainty calibration methods without external change indicators, where adaptation must be inferred from observed errors.



\newpage
{
\small
\bibliographystyle{abbrvnat}
\bibliography{refs}

@inproceedings{zhou2021informer,
  title={Informer: Beyond efficient transformer for long sequence time-series forecasting},
  author={Zhou, Haoyi and Zhang, Shanghang and Peng, Jieqi and Zhang, Shuai and Li, Jianxin and Xiong, Hui and Zhang, Wancai},
  booktitle={Proceedings of the AAAI conference on artificial intelligence},
  volume={35},
  pages={11106--11115},
  year={2021}
}

@inproceedings{lai2018modeling,
  title={Modeling long- and short-term temporal patterns with deep neural networks},
  author={Lai, Guokun and Chang, Wei-Cheng and Yang, Yiming and Liu, Hanxiao},
  booktitle={The 41st international ACM SIGIR conference on research \& development in information retrieval},
  pages={95--104},
  year={2018}
}

@article{wu2021autoformer,
  title={Autoformer: Decomposition transformers with auto-correlation for long-term series forecasting},
  author={Wu, Haixu and Xu, Jiehui and Wang, Jianmin and Long, Mingsheng},
  journal={Advances in Neural Information Processing Systems},
  volume={34},
  pages={22419--22430},
  year={2021}
}

@article{matteson2014nonparametric,
  title={A nonparametric approach for multiple change point analysis of multivariate data},
  author={Matteson, David S and James, Nicholas A},
  journal={Journal of the American Statistical Association},
  volume={109},
  number={505},
  pages={334--345},
  year={2014},
  publisher={Taylor \& Francis}
}

@article{dvoretzky1956asymptotic,
  title={Asymptotic minimax character of the sample distribution function and of the classical multinomial estimator},
  author={Dvoretzky, Aryeh and Kiefer, Jack and Wolfowitz, Jacob},
  journal={The Annals of Mathematical Statistics},
  pages={642--669},
  year={1956}
}

@article{massart1990tight,
  title={The tight constant in the Dvoretzky-Kiefer-Wolfowitz inequality},
  author={Massart, Pascal},
  journal={The Annals of Probability},
  pages={1269--1283},
  year={1990}
}

@article{cai2010optimal,
  title={Optimal rates of convergence for covariance matrix estimation},
  author={Cai, T Tony and Zhang, Cun-Hui and Zhou, Harrison H},
  journal={The Annals of Statistics},
  volume={38},
  number={4},
  pages={2118--2144},
  year={2010},
  publisher={Institute of Mathematical Statistics}
}

@inproceedings{
miyato2018spectral,
title={Spectral Normalization for Generative Adversarial Networks},
author={Takeru Miyato and Toshiki Kataoka and Masanori Koyama and Yuichi Yoshida},
booktitle={International Conference on Learning Representations},
year={2018}
}

@book{vershynin2018high,
  title={High-dimensional probability: An introduction with applications in data science},
  author={Vershynin, Roman},
  volume={47},
  year={2018},
  publisher={Cambridge University Press}
}

@article{gibbs2021aci,
  title={Adaptive conformal inference under distribution shift},
  author={Gibbs, Isaac and Candes, Emmanuel},
  journal={Advances in Neural Information Processing Systems},
  volume={34},
  pages={1660--1672},
  year={2021}
}

@book{wainwright2019high,
  title={High-dimensional statistics: A non-asymptotic viewpoint},
  author={Wainwright, Martin J},
  volume={48},
  year={2019},
  publisher={Cambridge University Press}
}

@article{xie2012change,
  title={Change-point detection for high-dimensional time series with missing data},
  author={Xie, Yao and Huang, Jiaji and Willett, Rebecca},
  journal={IEEE Journal of Selected Topics in Signal Processing},
  volume={7},
  number={1},
  pages={12--27},
  year={2013},
  publisher={IEEE}
}

@book{basseville1993detection,
  title={Detection of abrupt changes: theory and application},
  author={Basseville, Michele and Nikiforov, Igor V and others},
  volume={104},
  year={1993},
  publisher={Prentice Hall, Englewood Cliffs}
}

@article{barber2023conformal,
  title={Conformal prediction beyond exchangeability},
  author={Barber, Rina Foygel and Candes, Emmanuel J and Ramdas, Aaditya and Tibshirani, Ryan J},
  journal={The Annals of Statistics},
  volume={51},
  number={2},
  pages={816--845},
  year={2023},
  publisher={Institute of Mathematical Statistics}
}

@article{guan2023localized,
  title={Localized conformal prediction: A generalized inference framework for conformal prediction},
  author={Guan, Leying},
  journal={Biometrika},
  volume={110},
  number={1},
  pages={33--50},
  year={2023},
  publisher={Oxford University Press}
}

@inproceedings{bhatnagar2023online,
  title={Improved online conformal prediction via strongly adaptive online learning},
  author={Bhatnagar, Aadyot and Wang, Huan and Xiong, Caiming and Bai, Yu},
  booktitle={International Conference on Machine Learning},
  pages={2337--2363},
  year={2023},
  organization={PMLR}
}

@article{angelopoulos2023conformal,
  title={Conformal {PID} control for time series prediction},
  author={Angelopoulos, Anastasios and Candes, Emmanuel and Tibshirani, Ryan J},
  journal={Advances in Neural Information Processing Systems},
  volume={36},
  pages={23047--23074},
  year={2023}
}

@inproceedings{johnstone2025exact,
  title     = {Exact and Approximate Conformal Inference for Multi-Output Regression},
  author    = {Johnstone, Chancellor and Ndiaye, Eugene},
  booktitle = {Proceedings of the Fourteenth Symposium on Conformal and Probabilistic Prediction with Applications},
  organization={PMLR},
  volume    = {266},
  pages     = {153--172},
  year      = {2025}
}

@article{feldman2023conformal,
  title={Calibrated multiple-output quantile regression with representation learning},
  author={Feldman, Shai and Bates, Stephen and Romano, Yaniv},
  journal={Journal of Machine Learning Research},
  volume={24},
  number={24},
  pages={1--48},
  year={2023}
}

@inproceedings{rasul2021timegrad,
  title={Autoregressive denoising diffusion models for multivariate probabilistic time series forecasting},
  author={Rasul, Kashif and Seward, Calvin and Schuster, Ingmar and Vollgraf, Roland},
  booktitle={International Conference on Machine Learning},
  pages={8857--8868},
  year={2021},
  organization={PMLR}
}

@inproceedings{li2024tmdm,
  title={Transformer-modulated diffusion models for probabilistic multivariate time series forecasting},
  author={Li, Yuxin and Chen, Wenchao and Hu, Xinyue and Chen, Bo and Sun, Baolin and Zhou, Mingyuan},
  booktitle={The Twelfth International Conference on Learning Representations},
  year={2024}
}

@inproceedings{ye2025nsdiff,
  title = 	 {Non-stationary Diffusion For Probabilistic Time Series Forecasting},
  author =       {Ye, Weiwei and Xu, Zhuopeng and Gui, Ning},
  booktitle = 	 {Proceedings of the 42nd International Conference on Machine Learning},
  pages = 	 {72112--72130},
  year = 	 {2025},
  volume = 	 {267},
  organization={PMLR}
}

@article{salinas2020deepar,
  title={{DeepAR}: Probabilistic forecasting with autoregressive recurrent networks},
  author={Salinas, David and Flunkert, Valentin and Gasthaus, Jan and Januschowski, Tim},
  journal={International journal of forecasting},
  volume={36},
  number={3},
  pages={1181--1191},
  year={2020},
  publisher={Elsevier}
}

@inproceedings{tactis2,
title={{TACT}iS-2: Better, Faster, Simpler Attentional Copulas for Multivariate Time Series},
author={Arjun Ashok and {\'E}tienne Marcotte and Valentina Zantedeschi and Nicolas Chapados and Alexandre Drouin},
booktitle={The Twelfth International Conference on Learning Representations},
year={2024}
}

@article{salinas2019gpcopula,
  title={High-dimensional multivariate forecasting with low-rank Gaussian copula processes},
  author={Salinas, David and Bohlke-Schneider, Michael and Callot, Laurent and Medico, Roberto and Gasthaus, Jan},
  journal={Advances in Neural Information Processing Systems},
  volume={32},
  year={2019}
}

@inproceedings{kollovieh2024tsflow,
title={Flow Matching with Gaussian Process Priors for Probabilistic Time Series Forecasting},
author={Marcel Kollovieh and Marten Lienen and David L{\"u}dke and Leo Schwinn and Stephan G{\"u}nnemann},
booktitle={The Thirteenth International Conference on Learning Representations},
year={2025}
}

@inproceedings{xu2021enbpi,
  title = 	 {Conformal prediction interval for dynamic time-series},
  author =       {Xu, Chen and Xie, Yao},
  booktitle = 	 {Proceedings of the 38th International Conference on Machine Learning},
  pages = 	 {11559--11569},
  year = 	 {2021},
  volume = 	 {139},
  organization={PMLR},
  month = 	 {18--24 Jul}
}

@inproceedings{xu2024multidimspci,
  title = 	 {Conformal prediction for multi-dimensional time series by ellipsoidal sets},
  author =       {Xu, Chen and Jiang, Hanyang and Xie, Yao},
  booktitle = 	 {Proceedings of the 41st International Conference on Machine Learning},
  pages = 	 {55076--55099},
  year = 	 {2024},
  volume = 	 {235},
  organization={PMLR},
  month = 	 {21--27 Jul}
}

@article{tashiro2021csdi,
  title={{CSDI}: Conditional score-based diffusion models for probabilistic time series imputation},
  author={Tashiro, Yusuke and Song, Jiaming and Song, Yang and Ermon, Stefano},
  journal={Advances in Neural Information Processing Systems},
  volume={34},
  pages={24804--24816},
  year={2021}
}

@article{messoudi2021copula,
  title={Copula-based conformal prediction for multi-target regression},
  author={Messoudi, Soundouss and Destercke, S{\'e}bastien and Rousseau, Sylvain},
  journal={Pattern Recognition},
  volume={120},
  pages={108101},
  year={2021},
  publisher={Elsevier}
}

@inproceedings{sun2024copulacpts,
title={Copula Conformal prediction for multi-step time series prediction},
author={Sophia Huiwen Sun and Rose Yu},
booktitle={The Twelfth International Conference on Learning Representations},
year={2024}
}

@inproceedings{lee2026fcp,
title={Flow-based Conformal Prediction for Multi-dimensional Time Series},
author={Junghwan Lee and Chen Xu and Yao Xie},
booktitle={The Fourteenth International Conference on Learning Representations},
year={2026}
}

@book{vovk2005algorithmic,
  title={Algorithmic learning in a random world},
  author={Vovk, Vladimir and Gammerman, Alexander and Shafer, Glenn},
  year={2005},
  publisher={Springer}
}

@article{lei2018distribution,
  title={Distribution-free predictive inference for regression},
  author={Lei, Jing and G{'}Sell, Max and Rinaldo, Alessandro and Tibshirani, Ryan J and Wasserman, Larry},
  journal={Journal of the American Statistical Association},
  volume={113},
  number={523},
  pages={1094--1111},
  year={2018},
  publisher={Taylor \& Francis}
}

@article{angelopoulos2023gentle,
  title   = {A Gentle Introduction to Conformal Prediction and Distribution-Free Uncertainty Quantification},
  author  = {Angelopoulos, Anastasios N. and Bates, Stephen},
  journal = {arXiv preprint arXiv:2107.07511},
  year    = {2021}
}

@article{romano2019cqr,
  title={Conformalized quantile regression},
  author={Romano, Yaniv and Patterson, Evan and Candes, Emmanuel},
  journal={Advances in Neural Information Processing Systems},
  volume={32},
  year={2019}
}

@article{tibshirani2019weighted,
  title={Conformal prediction under covariate shift},
  author={Tibshirani, Ryan J and Foygel Barber, Rina and Candes, Emmanuel and Ramdas, Aaditya},
  journal={Advances in Neural Information Processing Systems},
  volume={32},
  year={2019}
}

@article{rangapuram2018deepstate,
  title={Deep state space models for time series forecasting},
  author={Rangapuram, Syama Sundar and Seeger, Matthias W and Gasthaus, Jan and Stella, Lorenzo and Wang, Yuyang and Januschowski, Tim},
  journal={Advances in Neural Information Processing Systems},
  volume={31},
  year={2018}
}

@inproceedings{rasul2020cnf,
title={Multivariate Probabilistic Time Series Forecasting via Conditioned Normalizing Flows},
author={Kashif Rasul and Abdul-Saboor Sheikh and Ingmar Schuster and Urs M Bergmann and Roland Vollgraf},
booktitle={International Conference on Learning Representations},
year={2021}
}

@inproceedings{marx2023distributionmatching,
  title     = {Calibration by Distribution Matching: Trainable Kernel Calibration Metrics},
  author    = {Marx, Charlie and Zalouk, Sofian and Ermon, Stefano},
  booktitle = {Advances in Neural Information Processing Systems},
  volume    = {36},
  pages     = {25910--25928},
  year      = {2023}
}

@inproceedings{kuleshov2022calibrated,
  title={Calibrated and sharp uncertainties in deep learning via density estimation},
  author={Kuleshov, Volodymyr and Deshpande, Shachi},
  booktitle={International Conference on Machine Learning},
  pages={11683--11693},
  year={2022},
  organization={PMLR}
}

@article{yang2026cp4gen,
  title={Conformal Prediction for Generative Models via Adaptive Cluster-Based Density Estimation},
  author={Yang, Qidong and Zhu, Qianyu Julie and Giezendanner, Jonathan and Marzouk, Youssef and Bates, Stephen and Wang, Sherrie},
  journal={arXiv preprint arXiv:2601.22298},
  year={2026}
}

@inproceedings{messoudi2022ellipsoidal,
  title={Ellipsoidal conformal inference for multi-target regression},
  author={Messoudi, Soundouss and Destercke, S{\'e}bastien and Rousseau, Sylvain},
  booktitle={Conformal and Probabilistic Prediction with Applications},
  pages={294--306},
  year={2022},
  organization={PMLR}
}

@article{zhao2021calibrating,
  title={Calibrating predictions to decisions: A novel approach to multi-class calibration},
  author={Zhao, Shengjia and Kim, Michael and Sahoo, Roshni and Ma, Tengyu and Ermon, Stefano},
  journal={Advances in Neural Information Processing Systems},
  volume={34},
  pages={22313--22324},
  year={2021}
}

@inproceedings{deng2025uncertainty,
  title     = {Uncertainty Quantification for Conditional Treatment Effect Estimation under Dynamic Treatment Regimes},
  author    = {Deng, Leon and Xiong, Hong and Wu, Feng and Kapoor, Sanyam and Gosh, Soumya and Shahn, Zach and Lehman, Li-wei},
  booktitle = {Proceedings of the 4th Machine Learning for Health Symposium},
  organization={PMLR},
  volume    = {259},
  pages     = {248--266},
  year      = {2025}
}
}

\newpage
\appendix

\section{Proofs for Section~\ref{sec:theory}}
\label{app:theory_proofs}

\subsection{Assumptions.}
We work under the following four assumptions. They are stated in the same order
and with the same meaning as in Section~\ref{sec:theory}. Let \(\mathcal F_t\)
be the natural filtration generated by realized targets, forecast sample clouds,
and all auxiliary randomization up to time \(t\). Let \(\tau_t<t\) denote the
last change point before prediction time \(t\). Let \(\mathbb Q_t(\cdot)\)
denote the probability measure conditional on the pre-regime information
\(\mathcal F_{\tau_t}\) and the deterministic window-design parameters, but
marginal over the realized augmented objects inside the current regime. Let
\(\mathbb E_{\mathbb Q_t}[\cdot]\) be the corresponding expectation. 

\begin{enumerate}
\item[(A1)] \textbf{Conditional sample-cloud regularity and stable score geometry.}
We assume that, at each prediction time $t$, the forecast sample cloud is conditionally i.i.d. and sufficiently light-tailed after standardization. We also assume that the Mahalanobis conformal set is locally stable with respect to covariance-estimation error, in a pathwise sense conditional on the realized covariance estimates. These regularity conditions are standard in high-dimensional probability and are adapted from the foundational framework established by Vershynin \cite{vershynin2018high}. More formally:

\emph{(A1a) Conditional sample-cloud law.}
Conditional on \(\mathcal F_{t-1}\),
\begin{equation}
\hat y_t^{(1)},\ldots,\hat y_t^{(M)}
\overset{\mathrm{i.i.d.}}{\sim}
\widehat P_t(\cdot\mid\mathcal F_{t-1}).
\label{eq:theory_sample_cloud_law}
\end{equation}
Let \(D=\operatorname{diag}(s_1,\ldots,s_d)\) be fixed before time \(t\), with
\[
0<s_{\min}\le s_j\le s_{\max}<\infty .
\]
For an \(\mathcal F_{t-1}\)-measurable center \(\mu_t^\star\), define
\[
\xi_t^{(m)}
=
D^{-1}\bigl(\hat y_t^{(m)}-\mu_t^\star\bigr),
\qquad m=1,\ldots,M.
\]

\emph{(A1b) Sub-Gaussian residuals and covariance estimation target.}
Conditional on \(\mathcal F_{t-1}\), the vectors
\(\xi_t^{(1)},\ldots,\xi_t^{(M)}\) are mean-zero, \(K\)-sub-Gaussian, and have
conditional covariance
\begin{equation}
C_t^\star
=
\operatorname{Cov}\!\left(\xi_t^{(m)}\mid\mathcal F_{t-1}\right).
\label{eq:population_covariance}
\end{equation}
The empirical covariance \(\widehat C_t\) in \eqref{eq:samplecov} is the
centered empirical covariance of the standardized sample cloud. Hence it is
unchanged by replacing the ideal center \(\mu_t^\star\) with the median-based
center \(\mu_t\) used in the method.

\emph{(A1c) Pathwise stability of the conformal set under covariance perturbation.}
For any deterministic or previsible same-regime calibration index set
\(\mathcal H_t\), and in particular for the oracle set \(\mathcal H_t^\star\),
let \(\mathcal C_t^\star(\alpha_t)\) denote the conformal set constructed from
ideal scores, and let \(\mathcal C_t(\alpha_t)\) denote the conformal set
constructed from empirical scores.

For such a set \(\mathcal H_t\), define the conditional law
\[
\widetilde{\mathbb Q}_{t,\mathcal H_t}(\cdot)
=
\mathbb P\!\left(
\cdot
\mid
\mathcal F_{\tau_t},
(\widehat C_s)_{s\in\mathcal H_t\cup\{t\}}
\right).
\]
There exist constants \(\eta_0>0\) and \(L_\Sigma<\infty\) such that, on the
event
\[
\max_{s\in\mathcal H_t\cup\{t\}}
\|\widehat C_s-C_s^\star\|_{\mathrm{op}}
\le
\eta
\le
\eta_0,
\]
we have
\begin{equation}
\left|
\widetilde{\mathbb Q}_{t,\mathcal H_t}\!\left(
Y_t\in\mathcal C_t(\alpha_t)
\right)
-
\widetilde{\mathbb Q}_{t,\mathcal H_t}\!\left(
Y_t\in\mathcal C_t^\star(\alpha_t)
\right)
\right|
\le
L_\Sigma\eta .
\label{eq:score_geometry_stability}
\end{equation}
This condition is implied, for example, by local stability of the stabilized
pseudo-inverse map together with an anti-concentration condition for the ideal
radial score distribution.

\item[(A2)] \textbf{Same-regime test calibration and contaminated-window detectability.}
A2 is a finite-sample version of two standard requirements in change-point and sequential goodness-of-fit testing: size control under homogeneous samples and power under separated alternatives. Fix a prediction time \(t\), and let \(\tau_t\) denote the most recent regime change before time \(t\). Thus, the current regime contains the observations
indexed by
\[
\{\tau_t+1,\ldots,t-1\}.
\]
Let \(\mathbb Q_t\) denote the probability law, or the relevant conditional
probability law, governing the data and random test statistics at time \(t\).
This assumption states that same-regime candidate blocks are not rejected too
often, while any candidate window that crosses a regime boundary is detected
with high probability. The assumption is stated directly at the level of the
finite-sample test errors used by the window-selection algorithm. This
calibration-versus-detectability form is standard in change-point analysis:
homogeneous samples require size control, while contaminated samples require a
separation condition large enough to overcome sampling noise
\cite{basseville1993detection,xie2012change}.

More formally:

\emph{(A2a) Oracle same-regime window.}
Let
\[
\mathcal G=\{L_{\min},L_{\min}+h,\ldots,L_{\max}\}
\]
be the finite grid of candidate window lengths, where \(L_{\min}\) and
\(L_{\max}\) are the minimum and maximum candidate lengths, respectively, and
\(h\) is the grid step size. Let \(\mathcal P_t\) denote the probe block at time
\(t\), namely the recent block used as the reference block for testing whether
a candidate block belongs to the same regime. Assume that \(\mathcal P_t\) lies
inside the current regime and that there exists at least one same-regime
candidate length. For each \(L\in\mathcal G\), let
\(\mathcal C_t^{\mathrm{whole}}(L)\) denote the whole candidate block of length
\(L\) used by the algorithm. Define the oracle same-regime length
\(L_t^\star\) as the largest candidate length whose whole candidate block and
probe block are both contained in the current regime:
\begin{equation}
L_t^\star
=
\max
\Bigl\{
L\in\mathcal G:
\mathcal C_t^{\mathrm{whole}}(L)\cup\mathcal P_t
\subseteq
\{\tau_t+1,\ldots,t-1\}
\Bigr\}.
\label{eq:oracle_same_regime_length}
\end{equation}

For each candidate block
\[
\mathcal A
\in
\{\mathcal C_t^{\mathrm{whole}}(L),\mathcal B_t^{\mathrm{inc}}(L)\},
\]
where \(\mathcal B_t^{\mathrm{inc}}(L)\) denotes the incremental block
associated with candidate length \(L\), and for each diagnostic
\(j\in\{1,2,3\}\), let \(\mathrm{KS}_{t,L,j}^{\mathcal A}\) denote the
empirical KS-type discrepancy in \eqref{eq:same_regime_test}. Here \(j\) indexes
the diagnostic being tested, \(L\) indexes the candidate length, and
\(\mathcal A\) indexes the tested candidate block. Let
\(\tau_{t,L,j}^{\mathcal A}\) denote the corresponding rejection threshold for
this KS-type test. The thresholds are deterministic functions of the
window-design parameters, sample sizes, and nominal test levels, or are
measurable with respect to information outside the realized test statistics.
For the KS-type diagnostics used here, such thresholds can be justified by the
Dvoretzky--Kiefer--Wolfowitz inequality and Massart's sharp constant, which give
finite-sample uniform bounds for empirical distribution functions
\cite{dvoretzky1956asymptotic,massart1990tight}.

\emph{(A2b) Clean-block calibration.}
For every clean tested block, meaning every tested block \(\mathcal A\) whose
union with the probe block \(\mathcal P_t\) lies entirely inside the current
regime,
\[
\mathcal A\cup\mathcal P_t
\subseteq
\{\tau_t+1,\ldots,t-1\},
\]
the probability of false rejection is controlled:
\begin{equation}
\mathbb Q_t
\left(
\mathrm{KS}_{t,L,j}^{\mathcal A}
>
\tau_{t,L,j}^{\mathcal A}
\right)
\le
\delta_{\mathrm{uni}}.
\label{eq:clean_block_test_control}
\end{equation}
Here \(\delta_{\mathrm{uni}}\) is the per-test, or union-adjusted, upper bound
on the false rejection probability. This is the finite-sample size-control
requirement for the same-regime test. When \(\mathcal A\) and \(\mathcal P_t\)
are drawn from the same regime, the population discrepancy is zero, and the
empirical discrepancy is controlled by standard empirical-process
concentration; the constant \(\delta_{\mathrm{uni}}\) can also absorb a union
bound over the finite set of candidate lengths, blocks, and diagnostics
\cite{dvoretzky1956asymptotic,massart1990tight}.

\emph{(A2c) Contaminated-window detectability.}
For every \(L>L_t^\star\), the candidate length \(L\) is longer than the oracle
same-regime length and hence its associated candidate window necessarily
crosses a regime boundary. The assumption requires that at least one tested pair
\[
(\mathcal A,j),
\qquad
\mathcal A
\in
\{\mathcal C_t^{\mathrm{whole}}(L),\mathcal B_t^{\mathrm{inc}}(L)\},
\quad
j\in\{1,2,3\},
\]
detects this contamination. Specifically, the probability that this
contaminated tested pair fails to reject is bounded by
\(\beta_{\mathrm{det}}\):
\begin{equation}
\mathbb Q_t
\left(
\mathrm{KS}_{t,L,j}^{\mathcal A}
\le
\tau_{t,L,j}^{\mathcal A}
\right)
\le
\beta_{\mathrm{det}}.
\label{eq:contaminated_block_test_control}
\end{equation}
Here \(\beta_{\mathrm{det}}\) is the upper bound on the missed-detection
probability. This is the corresponding finite-sample power requirement. It is a
local identifiability condition: if a candidate window crosses a regime
boundary, then at least one of the diagnostics must see a population
distributional shift large enough to be detected with probability at least
\(1-\beta_{\mathrm{det}}\), as in standard abrupt-change and high-dimensional
change-point formulations \cite{basseville1993detection,xie2012change}.

A sufficient, but not necessary, condition for
\eqref{eq:contaminated_block_test_control} is the following. For some tested
pair \((\mathcal A,j)\), let \(H_{t,L,j}^{\mathcal A}\) denote the population
distribution function, or population probability-integral-transform target, of
the diagnostic values comparing the candidate block \(\mathcal A\) with the
probe block \(\mathcal P_t\). Let \(\mathrm{Id}(x)=x\) denote the identity
distribution function on \([0,1]\), which corresponds to the ideal same-regime
case after probability-integral transformation. Let
\(\widehat F_{t,L,j}^{\mathcal A}\) denote the empirical distribution function
estimated from the probe block, and let \(\Delta_{\mathrm{sep}}>0\) denote the
separation margin between the population discrepancy and the test threshold.
Suppose
\[
\|H_{t,L,j}^{\mathcal A}-\mathrm{Id}\|_\infty
\ge
\tau_{t,L,j}^{\mathcal A}
+
\Delta_{\mathrm{sep}},
\]
where \(\|\cdot\|_\infty\) denotes the uniform sup-norm distance, and suppose
that the empirical distribution function satisfies the concentration bound
\[
\mathbb Q_t
\left(
\|\widehat F_{t,L,j}^{\mathcal A}-H_{t,L,j}^{\mathcal A}\|_\infty
>
\Delta_{\mathrm{sep}}
\right)
\le
2\exp\{-2|\mathcal P_t|\Delta_{\mathrm{sep}}^2\}.
\]
Here \(|\mathcal P_t|\) is the number of observations in the probe block. Then
\eqref{eq:contaminated_block_test_control} holds with
\[
\beta_{\mathrm{det}}
=
2\exp\{-2|\mathcal P_t|\Delta_{\mathrm{sep}}^2\}.
\]
Indeed, the population separation margin \(\Delta_{\mathrm{sep}}\) ensures that
the empirical KS statistic can fall below the threshold only if the empirical
process deviates from its population target by more than
\(\Delta_{\mathrm{sep}}\). The displayed exponential bound is therefore exactly
the DKW/Massart-type tail probability for this failure event
\cite{dvoretzky1956asymptotic,massart1990tight}.

\item[(A3)] \textbf{Piecewise augmented exchangeability.}
This assumption imposes the standard exchangeability condition used in
conformal prediction, but only locally within the current regime. Specifically,
conditional on the pre-regime information \(\mathcal F_{\tau_t}\), where
\(\tau_t\) is the most recent regime change before time \(t\), the observations
inside the current regime are treated as exchangeable. This is the classical
condition under which conformal methods obtain finite-sample rank validity and
distribution-free coverage \cite{vovk2005algorithmic,lei2018distribution,angelopoulos2023gentle}.
Here the condition is imposed piecewise, after conditioning on the regime start,
which is natural in nonstationary settings where exchangeability is expected to
hold only locally rather than over the entire time series.

More formally, for each \(s=\tau_t+1,\ldots,t\), define the augmented
forecast-target object
\[
\mathcal Z_s=(Y_s,\widehat{\mathcal Y}_s),
\]
where \(Y_s\) is the realized target and \(\widehat{\mathcal Y}_s\) is the
corresponding forecast object, such as a predictive distribution, trajectory
forecast, or prediction set. The augmented objects
\[
\mathcal Z_{\tau_t+1},\ldots,\mathcal Z_t
\]
are conditionally exchangeable given \(\mathcal F_{\tau_t}\). That is, for any
finite permutation \(\pi\) of \(\{\tau_t+1,\ldots,t\}\),
\[
(\mathcal Z_{\tau_t+1},\ldots,\mathcal Z_t)
\mid \mathcal F_{\tau_t}
\overset{d}{=}
(\mathcal Z_{\pi(\tau_t+1)},\ldots,\mathcal Z_{\pi(t)})
\mid \mathcal F_{\tau_t}.
\]
This is the regime-local analogue of the exchangeability assumption commonly
used in split conformal prediction and related finite-sample conformal
procedures \cite{vovk2005algorithmic,lei2018distribution,romano2019cqr}.

The diagnostics and ideal conformity scores are assumed to be generated from
these augmented objects by common measurable permutation-equivariant maps,
possibly together with auxiliary tie-breaking random variables that are
independent of the augmented objects conditional on \(\mathcal F_{\tau_t}\).
Permutation equivariance means that if the augmented objects are permuted, the
resulting diagnostics or scores are permuted in the same way, rather than
changing their joint distribution. Therefore, by the standard closure property
of exchangeability under common measurable equivariant transformations, the
ideal score collection for any same-regime index set \(\mathcal H_t\),
\[
\{e_s^\star:s\in\mathcal H_t\}\cup\{e_t^\star\},
\]
is conditionally exchangeable under \(\mathbb Q_t\). This is the exact property
needed in the conformal coverage proof: the test score \(e_t^\star\) has a
uniform rank among the same-regime calibration scores, up to the usual
tie-breaking convention \cite{vovk2005algorithmic,angelopoulos2023gentle}.

\item[(A4)] \textbf{ACI adaptation error and rank validity.}
Let \(\mathcal C_{t,\mathrm{or}}^\star(\alpha_t)\) denote the oracle
same-regime conformal set constructed from the ideal scores on
\(\mathcal H_t^\star\) using the adaptive miscoverage level
\(\alpha_t\in[0,1]\), and let \(n_t^\star=|\mathcal H_t^\star|\) be the number
of oracle same-regime calibration scores. This assumption separates the effect
of adaptive level selection from the usual conformal rank-validity argument.
Such a separation is standard in adaptive conformal inference: the adaptive
procedure controls or tracks the target miscoverage level over time, while the
conformal coverage argument relies on the rank of the test score among the
calibration scores \cite{gibbs2021aci,angelopoulos2023conformal,bhatnagar2023online}.

\emph{(A4a) Mean tracking error.}
The adaptive level remains close to the nominal level in expectation:
\begin{equation}
\mathbb E_{\mathbb Q_t}\!\left[|\alpha_t-\alpha|\right]
\le
\varepsilon_t^{\mathrm{ACI}} .
\label{eq:aci_tracking_error}
\end{equation}
Here \(\alpha\) is the target miscoverage level, \(\alpha_t\) is the
data-adaptive level selected by the ACI update rule at time \(t\), and
\(\varepsilon_t^{\mathrm{ACI}}\) summarizes the finite-sample adaptation or
tracking error of this update. This type of condition is the standard way to
abstract the performance of an adaptive conformal update under distribution
shift: rather than requiring \(\alpha_t=\alpha\) exactly at every time point,
one allows the online update to deviate from the target level by a controlled
amount \cite{gibbs2021aci,angelopoulos2023conformal}. Strongly adaptive
online-learning formulations provide one route for bounding such tracking errors
over nonstationary sequences \cite{bhatnagar2023online}.

\emph{(A4b) Rank validity under adaptivity.}
Conditional on \(\alpha_t\), the rank of the test ideal score among the
\(n_t^\star+1\) oracle ideal scores
\[
\{e_s^\star:s\in\mathcal H_t^\star\}\cup\{e_t^\star\}
\]
is uniformly distributed. Equivalently, the oracle ideal conformal predictor
satisfies the finite-sample rank bound
\begin{equation}
\left|
\mathbb Q_t\!\left(
Y_t\in \mathcal C_{t,\mathrm{or}}^\star(\alpha_t)
\right)
-
\left(1-\mathbb E_{\mathbb Q_t}[\alpha_t]\right)
\right|
\le
\frac{1}{n_t^\star+1}.
\label{eq:aci_rank_validity}
\end{equation}
This is the usual finite-sample conformal rank argument, applied after
conditioning on the adaptive level. It is satisfied, for example, when
\(\alpha_t\) is predictable with respect to the score-ranking step, meaning
that it is determined by past information, previous coverage errors, or other
external randomness, but not by the relative rank of \(e_t^\star\) among the
same-regime calibration scores at time \(t\). Under this previsible-level
condition, conditioning on \(\alpha_t\) does not destroy the exchangeability of
the oracle ideal scores, so the test score keeps the usual uniform-rank
property \cite{vovk2005algorithmic,lei2018distribution,angelopoulos2023gentle}.
The additional \(1/(n_t^\star+1)\) term is the standard finite-sample
discretization error from using \(n_t^\star\) calibration scores in conformal
prediction.

Combining \eqref{eq:aci_tracking_error} and \eqref{eq:aci_rank_validity}, we have
\begin{equation}
\left|
\mathbb Q_t\!\left(
Y_t\in \mathcal C_{t,\mathrm{or}}^\star(\alpha_t)
\right)
-
(1-\alpha)
\right|
\le
\varepsilon_t^{\mathrm{ACI}}
+
\frac{1}{n_t^\star+1}.
\label{eq:aci_oracle_coverage_error}
\end{equation}
Thus, A4 is not an additional distributional assumption beyond the local
exchangeability condition in A3. Rather, it records two standard ingredients
needed when conformal prediction is combined with an adaptive level: a bounded
ACI tracking error and the preservation of the conformal rank argument under a
previsible adaptive level.

\end{enumerate}

\subsection{Proof of Proposition~\ref{prop:sample_cov_concentration}}

Our goal is to establish the concentration of the empirical sample-cloud covariance. While this is a foundational step for our method, the proof relies on well-established techniques rather than novel probabilistic machinery. Specifically, we follow the standard high-dimensional probability pipeline detailed in Vershynin \cite{vershynin2018high}.

\begin{proof}
Fix a prediction time $t$ and condition on $\mathcal F_{t-1}$. 
Write $\mathbb P_t(\cdot)=\mathbb P(\cdot\mid\mathcal F_{t-1})$ and 
$\mathbb E_t[\cdot]=\mathbb E[\cdot\mid\mathcal F_{t-1}]$. 
For brevity, drop the time subscript and write
\[
\xi_m = D^{-1}\bigl(\hat y_t^{(m)}-\mu_t^\star\bigr),
\qquad m=1,\ldots,M.
\]
By Assumption~(A1), conditional on $\mathcal F_{t-1}$, the vectors
$\xi_1,\ldots,\xi_M$ are independent, mean-zero, and sub-Gaussian with
sub-Gaussian parameter at most $K$, and
\[
C_t^\star=\mathbb E_t[\xi_1\xi_1^\top].
\]

Let
\[
\bar\xi=\frac1M\sum_{m=1}^M \xi_m,
\qquad
\widehat C_t
=
\frac1M\sum_{m=1}^M
(\xi_m-\bar\xi)(\xi_m-\bar\xi)^\top.
\]
Then
\[
\widehat C_t
=
\frac1M\sum_{m=1}^M \xi_m\xi_m^\top
-
\bar\xi\bar\xi^\top.
\]
Therefore,
\begin{equation}
\|\widehat C_t-C_t^\star\|_{\mathrm{op}}
\le
\left\|
\frac1M\sum_{m=1}^M \xi_m\xi_m^\top
-
C_t^\star
\right\|_{\mathrm{op}}
+
\|\bar\xi\|_2^2.
\label{eq:centered_cov_decomp}
\end{equation}

We first bound the uncentered empirical second-moment term. For any symmetric matrix $A$, if $\mathcal N$ is a $1/4$-net of $\mathcal S^{d-1}$, then
\[
\|A\|_{\mathrm{op}}
\le
2\max_{v\in\mathcal N}|v^\top A v|.
\]
Choose such a net with cardinality $|\mathcal N|\le 9^d$. Applying this to
\[
A=
\frac1M\sum_{m=1}^M \xi_m\xi_m^\top
-
C_t^\star,
\]
we obtain
\begin{equation}
\left\|
\frac1M\sum_{m=1}^M \xi_m\xi_m^\top
-
C_t^\star
\right\|_{\mathrm{op}}
\le
2\max_{v\in\mathcal N}
\left|
\frac1M\sum_{m=1}^M
\langle \xi_m,v\rangle^2
-
\mathbb E_t[\langle \xi_1,v\rangle^2]
\right|.
\label{eq:net_second_moment}
\end{equation}

Fix $v\in\mathcal N$. Since $\xi_m$ is conditionally $K$-sub-Gaussian,
$\langle \xi_m,v\rangle$ is conditionally $K$-sub-Gaussian. Hence
\[
Z_{m,v}
=
\langle \xi_m,v\rangle^2
-
\mathbb E_t[\langle \xi_m,v\rangle^2]
\]
is conditionally centered sub-exponential with sub-exponential norm bounded by
$c_1K^2$, for an absolute constant $c_1>0$. Bernstein's inequality for
independent centered sub-exponential random variables gives, for every $u>0$,
\[
\mathbb P_t\left(
\left|
\frac1M\sum_{m=1}^M Z_{m,v}
\right|
\ge u
\right)
\le
2\exp\left[
-c_2M
\min\left(
\frac{u^2}{K^4},
\frac{u}{K^2}
\right)
\right],
\]
where $c_2>0$ is an absolute constant.

Taking a union bound over $v\in\mathcal N$ yields
\[
\mathbb P_t\left(
\max_{v\in\mathcal N}
\left|
\frac1M\sum_{m=1}^M Z_{m,v}
\right|
\ge u
\right)
\le
2\exp\left[
d\log 9
-
c_2M
\min\left(
\frac{u^2}{K^4},
\frac{u}{K^2}
\right)
\right].
\]
Thus, for a sufficiently large absolute constant $c_3>0$, with
$\mathbb P_t$-probability at least $1-\delta/2$,
\begin{equation}
\left\|
\frac1M\sum_{m=1}^M \xi_m\xi_m^\top
-
C_t^\star
\right\|_{\mathrm{op}}
\le
c_3K^2
\left(
\sqrt{\frac{d+\log(1/\delta)}{M}}
+
\frac{d+\log(1/\delta)}{M}
\right).
\label{eq:uncentered_second_moment_bound}
\end{equation}

It remains to bound the centering term $\|\bar\xi\|_2^2$. For any fixed
$v\in\mathcal S^{d-1}$, the scalar random variable
$\langle \bar\xi,v\rangle$ is conditionally sub-Gaussian with parameter at most
$K/\sqrt M$. Applying the same $1/4$-net argument and a union bound gives that,
with $\mathbb P_t$-probability at least $1-\delta/2$,
\begin{equation}
\|\bar\xi\|_2^2
\le
c_4K^2
\frac{d+\log(1/\delta)}{M},
\label{eq:sample_mean_bound}
\end{equation}
for an absolute constant $c_4>0$.

Combining \eqref{eq:centered_cov_decomp},
\eqref{eq:uncentered_second_moment_bound}, and
\eqref{eq:sample_mean_bound}, and absorbing constants into a universal
constant $c_0>0$, we obtain that, with conditional probability at least
$1-\delta$,
\[
\|\widehat C_t-C_t^\star\|_{\mathrm{op}}
\le
c_0K^2
\left(
\sqrt{\frac{d+\log(1/\delta)}{M}}
+
\frac{d+\log(1/\delta)}{M}
\right).
\]
Since the argument holds conditionally on $\mathcal F_{t-1}$, the same
high-probability bound also holds marginally. This completes the proof.
\end{proof}

\paragraph{Discussion on the Tightness of the Covariance Bound.}
The high-probability bound established in Proposition~\ref{prop:sample_cov_concentration} is statistically tight and cannot be substantially improved without imposing strong structural assumptions on the generative model's output \cite{cai2010optimal}. 

Specifically, estimating an arbitrary $d \times d$ covariance matrix in the operator norm requires $\Omega(\sqrt{d/M})$ error, as established by standard minimax lower bounds via Fano's inequality \cite{cai2010optimal}. In the standard operational regime of \textsc{Space}, where the number of generated samples $M$ is reasonably large relative to the feature dimension $d$ (i.e., $M \ge d$), the linear dimension term $\mathcal{O}(d/M)$ is negligible, and our estimator achieves the optimal $\mathcal{O}(\sqrt{d/M})$ minimax rate.

This bound would only be considered loose if the true conditional covariance $C_t^\star$ possessed known, exploitable structure---such as high sparsity or strict low-rank constraints. In such cases, thresholded or PCA-based estimators could replace $\widehat C_t$ to achieve rates depending on the effective intrinsic dimension rather than the ambient dimension $d$. However, because \textsc{Space} operates on black-box generative models where the latent geometry of the predictive sample cloud is not known \textit{a priori}, the unstructured minimax optimal rate achieved here represents the fundamental statistical limit of the local geometry estimation step.

\subsection{Proof of Proposition~\ref{prop:window_recovery}}

\begin{proof}
Fix a prediction time \(t\). We work under the conditional probability measure
\(\mathbb Q_t\) specified in Assumption~(A2). Let \(J=|\mathcal G|\).

Define the collection of clean tests
\[
\mathcal T_{\mathrm{clean}}
=
\left\{
(L,\mathcal A,j):
L\le L_t^\star,\,
\mathcal A\in
\{\mathcal C_t^{\mathrm{whole}}(L),\mathcal B_t^{\mathrm{inc}}(L)\},\,
j\in\{1,2,3\}
\right\}.
\]
By the definition of \(L_t^\star\) and the nesting of the backward candidate
windows, for every \(L\le L_t^\star\),
\[
\mathcal C_t^{\mathrm{whole}}(L)
\cup
\mathcal B_t^{\mathrm{inc}}(L)
\cup
\mathcal P_t
\subseteq
\{\tau_t+1,\ldots,t-1\}.
\]
Therefore every test in \(\mathcal T_{\mathrm{clean}}\) is a clean-block test.
By the clean-block calibration part of Assumption~(A2),
\[
\mathbb Q_t
\left(
\mathrm{KS}_{t,L,j}^{\mathcal A}
>
\tau_{t,L,j}^{\mathcal A}
\right)
\le
\delta_{\mathrm{uni}}
\]
for every \((L,\mathcal A,j)\in\mathcal T_{\mathrm{clean}}\). Since
\(|\mathcal T_{\mathrm{clean}}|\le 6J\), the union bound gives
\[
\mathbb Q_t
\left(
\exists (L,\mathcal A,j)\in\mathcal T_{\mathrm{clean}}
:
\mathrm{KS}_{t,L,j}^{\mathcal A}
>
\tau_{t,L,j}^{\mathcal A}
\right)
\le
6J\delta_{\mathrm{uni}}.
\]

Next consider any candidate length \(L>L_t^\star\). By the contaminated-window
detectability part of Assumption~(A2), there exists a tested pair
\[
(\mathcal A_L,j_L),
\qquad
\mathcal A_L\in
\{\mathcal C_t^{\mathrm{whole}}(L),\mathcal B_t^{\mathrm{inc}}(L)\},
\quad
j_L\in\{1,2,3\},
\]
such that
\[
\mathbb Q_t
\left(
\mathrm{KS}_{t,L,j_L}^{\mathcal A_L}
\le
\tau_{t,L,j_L}^{\mathcal A_L}
\right)
\le
\beta_{\mathrm{det}}.
\]
The selector accepts candidate length \(L\) only if all block-diagnostic checks
for that length pass. Hence acceptance of \(L\) implies that the detectably
contaminated test \((\mathcal A_L,j_L)\) passes. Therefore,
\[
\mathbb Q_t
\left(
\text{\(\textsc{Space}\) accepts candidate length }L
\right)
\le
\beta_{\mathrm{det}}.
\]
Taking a union bound over all invalid candidate lengths \(L>L_t^\star\), of
which there are at most \(J\), yields
\[
\mathbb Q_t
\left(
\exists L>L_t^\star:
\text{\(\textsc{Space}\) accepts candidate length }L
\right)
\le
J\beta_{\mathrm{det}}.
\]

Let \(\mathcal E_t\) be the event that no clean test is falsely rejected and no
invalid candidate length is accepted. Combining the two union bounds,
\[
\mathbb Q_t(\mathcal E_t^c)
\le
6J\delta_{\mathrm{uni}}
+
J\beta_{\mathrm{det}}.
\]
On \(\mathcal E_t\), every candidate length \(L\le L_t^\star\) passes all
required checks, while every candidate length \(L>L_t^\star\) is not accepted.
Since the selector returns the longest accepted candidate length,
\[
\widehat L_t=L_t^\star
\quad
\text{on }\mathcal E_t.
\]
Therefore,
\[
\mathbb Q_t(\widehat L_t=L_t^\star)
\ge
1
-
6|\mathcal G|\delta_{\mathrm{uni}}
-
|\mathcal G|\beta_{\mathrm{det}}.
\]

If the sufficient empirical-process condition in Assumption~(A2) holds with
\[
\beta_{\mathrm{det}}
=
2\exp\{-2|\mathcal P_t|\Delta_{\mathrm{sep}}^2\},
\]
then
\[
\mathbb Q_t(\widehat L_t=L_t^\star)
\ge
1
-
6|\mathcal G|\delta_{\mathrm{uni}}
-
2|\mathcal G|
\exp\{-2|\mathcal P_t|\Delta_{\mathrm{sep}}^2\}.
\]
This completes the proof.
\end{proof}

\subsection{Proof of Theorem~\ref{thm:per_time_coverage}}

\begin{lemma}[Exchangeability of oracle ideal radial scores]
\label{lem:radial_score_exchangeability}
Let \(\mathcal H_t^\star\) be the oracle same-regime calibration set associated
with \(L_t^\star\). Under Assumption~(A3), the ideal score collection
\[
\{e_s^\star:s\in\mathcal H_t^\star\}\cup\{e_t^\star\}
\]
is conditionally exchangeable under \(\mathbb Q_t\).
\end{lemma}

\begin{proof}
By the definition of \(L_t^\star\), the indices
\(\mathcal H_t^\star\cup\{t\}\) lie in the same regime. Assumption~(A3)
directly gives conditional exchangeability of the ideal score collection for
any same-regime index set. Taking
\(\mathcal H_t=\mathcal H_t^\star\) proves the claim.
\end{proof}

\begin{proof}[Proof of Theorem~\ref{thm:per_time_coverage}]
Fix a prediction time \(t\), and let \(n_t^\star=|\mathcal H_t^\star|\).
Let \(\mathcal C_{t,\mathrm{or}}^\star(\alpha_t)\) denote the oracle
same-regime conformal set constructed from ideal scores on
\(\mathcal H_t^\star\), and let \(\mathcal C_{t,\mathrm{or}}(\alpha_t)\)
denote the corresponding oracle set constructed from empirical scores.
Let \(\mathcal C_t^{\mathrm{ACI}}(\alpha)\) denote the actual set returned by
the algorithm when the nominal target is \(\alpha\), internally using
\(\alpha_t\).

\paragraph{Step 1: Oracle ideal conformal coverage.}
By Lemma~\ref{lem:radial_score_exchangeability} and Assumption~(A4b), conditional
on \(\alpha_t\), the rank of \(e_t^\star\) among the \(n_t^\star+1\) oracle
ideal scores is uniform. Hence the split-conformal rank argument gives
\[
\left|
\mathbb Q_t\!\left(
Y_t\in\mathcal C_{t,\mathrm{or}}^\star(\alpha_t)
\right)
-
\left(1-\mathbb E_{\mathbb Q_t}[\alpha_t]\right)
\right|
\le
\frac{1}{n_t^\star+1}.
\]

\paragraph{Step 2: Covariance perturbation.}
Define
\[
\mathcal G_t^{\mathrm{cov}}
=
\left\{
\max_{s\in\mathcal H_t^\star\cup\{t\}}
\|\widehat C_s-C_s^\star\|_{\mathrm{op}}
\le
\eta_{M,\delta_t/(n_t^\star+1)}
\right\}.
\]
By Proposition~\ref{prop:sample_cov_concentration}, iterated expectation, and
a union bound over the \(n_t^\star+1\) time points,
\[
\mathbb Q_t\!\left((\mathcal G_t^{\mathrm{cov}})^c\right)
\le
\delta_t .
\]
Since \(\eta_{M,\delta_t/(n_t^\star+1)}\le\eta_0\), Assumption~(A1c) applies on
\(\mathcal G_t^{\mathrm{cov}}\). Thus replacing ideal scores by empirical
scores changes the oracle conformal coverage probability by at most
\(L_\Sigma\eta_{M,\delta_t/(n_t^\star+1)}\). Accounting for
\((\mathcal G_t^{\mathrm{cov}})^c\),
\[
\left|
\mathbb Q_t\!\left(
Y_t\in\mathcal C_{t,\mathrm{or}}(\alpha_t)
\right)
-
\mathbb Q_t\!\left(
Y_t\in\mathcal C_{t,\mathrm{or}}^\star(\alpha_t)
\right)
\right|
\le
L_\Sigma\eta_{M,\delta_t/(n_t^\star+1)}
+
\delta_t .
\]

\paragraph{Step 3: Window-selection error.}
Let \(\mathcal E_t=\{\widehat L_t=L_t^\star\}\). By
Proposition~\ref{prop:window_recovery},
\[
\mathbb Q_t(\mathcal E_t^c)
\le
\rho_t^{\mathrm{win}}.
\]
On \(\mathcal E_t\), the selected empirical calibration window equals the
oracle empirical calibration window, so
\[
\mathcal C_t^{\mathrm{ACI}}(\alpha)
=
\mathcal C_{t,\mathrm{or}}(\alpha_t).
\]
Therefore,
\[
\left|
\mathbb Q_t\!\left(
Y_t\in\mathcal C_t^{\mathrm{ACI}}(\alpha)
\right)
-
\mathbb Q_t\!\left(
Y_t\in\mathcal C_{t,\mathrm{or}}(\alpha_t)
\right)
\right|
\le
\mathbb Q_t(\mathcal E_t^c)
\le
\rho_t^{\mathrm{win}}.
\]

\paragraph{Step 4: ACI tracking.}
Combining Steps 1--3,
\[
\begin{aligned}
\left|
\mathbb Q_t\!\left(
Y_t\in\mathcal C_t^{\mathrm{ACI}}(\alpha)
\right)
-
\left(1-\mathbb E_{\mathbb Q_t}[\alpha_t]\right)
\right|
\le\;&
\rho_t^{\mathrm{win}}
+
\frac{1}{n_t^\star+1} \\
&+
L_\Sigma\eta_{M,\delta_t/(n_t^\star+1)}
+
\delta_t .
\end{aligned}
\]
By Assumption~(A4a),
\[
\left|
\left(1-\mathbb E_{\mathbb Q_t}[\alpha_t]\right)
-
(1-\alpha)
\right|
\le
\mathbb E_{\mathbb Q_t}[|\alpha_t-\alpha|]
\le
\varepsilon_t^{\mathrm{ACI}}.
\]
The triangle inequality yields
\[
\left|
\mathbb Q_t\!\left(
Y_t\in\mathcal C_t^{\mathrm{ACI}}(\alpha)
\right)
-
(1-\alpha)
\right|
\le
\rho_t^{\mathrm{win}}
+
\varepsilon_t^{\mathrm{ACI}}
+
\frac{1}{n_t^\star+1}
+
L_\Sigma\eta_{M,\delta_t/(n_t^\star+1)}
+
\delta_t .
\]
This completes the proof.
\end{proof}

\subsection{Proof of Corollary~\ref{cor:overall_gap}}

\begin{proof}
The empirical average coverage is
\[
\widehat{\mathrm{Cov}}_T
=
\frac1T\sum_{t=1}^T
\mathbf 1\{Y_t\in\mathcal C_t^{\mathrm{ACI}}(\alpha)\}.
\]
By the tower property,
\[
\mathbb E[\widehat{\mathrm{Cov}}_T]
=
\frac1T\sum_{t=1}^T
\mathbb E\!\left[
\mathbb Q_t\!\left(
Y_t\in\mathcal C_t^{\mathrm{ACI}}(\alpha)
\right)
\right].
\]
Hence
\[
\begin{aligned}
\left|
\mathbb E[\widehat{\mathrm{Cov}}_T]
-
(1-\alpha)
\right|
&\le
\frac1T\sum_{t=1}^T
\mathbb E\!\left[
\left|
\mathbb Q_t\!\left(
Y_t\in\mathcal C_t^{\mathrm{ACI}}(\alpha)
\right)
-
(1-\alpha)
\right|
\right].
\end{aligned}
\]
Applying Theorem~\ref{thm:per_time_coverage} for each \(t\) gives
\[
\left|
\mathbb E[\widehat{\mathrm{Cov}}_T]
-
(1-\alpha)
\right|
\le
\frac1T\sum_{t=1}^T
\left(
\rho_t^{\mathrm{win}}
+
\varepsilon_t^{\mathrm{ACI}}
+
\frac{1}{n_t^\star+1}
+
L_\Sigma\eta_{M,\delta_t/(n_t^\star+1)}
+
\delta_t
\right),
\]
provided these per-time terms are deterministic or hold almost surely under
the outer expectation. If they are random, the same display holds with outer
expectations applied to each right-hand-side term. Defining
\[
\bar\rho_T^{\mathrm{win}}
=
\frac1T\sum_{t=1}^T\rho_t^{\mathrm{win}},
\qquad
\bar\varepsilon_T^{\mathrm{ACI}}
=
\frac1T\sum_{t=1}^T\varepsilon_t^{\mathrm{ACI}},
\]
\[
\bar\eta_T
=
\frac1T\sum_{t=1}^T
\eta_{M,\delta_t/(n_t^\star+1)},
\qquad
\bar\delta_{T}
=
\frac1T\sum_{t=1}^T\delta_t,
\]
we obtain
\[
\left|
\mathbb E[\widehat{\mathrm{Cov}}_T]
-
(1-\alpha)
\right|
\le
\bar\rho_T^{\mathrm{win}}
+
\bar\varepsilon_T^{\mathrm{ACI}}
+
\frac1T\sum_{t=1}^T\frac1{n_t^\star+1}
+
L_\Sigma\bar\eta_T
+
\bar\delta_{T}.
\]
This completes the proof.
\end{proof}

\section{Algorithm}
\label{app:appendix_algo}

\begin{algorithm}[H]
\caption{Sequential \textsc{Space} wrapper}
\label{alg:space}
\begin{algorithmic}[1]
\Require Predictive sample stream $\{\widehat{\mathcal Y}_t\}_{t=1}^T$, observations revealed sequentially, target miscoverage $\alpha$, training scales $D$, probe length $p$, window range $[L_{\min},L_{\max}]$, step size $h$
\State Initialize internal level $\alpha_t\gets\alpha$
\State Initialize score and diagnostic histories
\For{$t=t_{\mathrm{cal}},\dots,T$}
    \State Compute center $\mu_t$ from $\widehat{\mathcal Y}_t$
    \State Form predictive residual samples $r_t^{(m)}=\hat y_t^{(m)}-\mu_t$
    \State Standardize $\tilde r_t^{(m)}=D^{-1}r_t^{(m)}$
    \State Compute $\widehat C_t=\operatorname{Cov}(\tilde r_t^{(1)},\dots,\tilde r_t^{(M)})$ and set $\widehat\Sigma_t=D\widehat C_tD$
    \State Form the regularized inverse $\widehat\Sigma_t^{+}$
    \State Construct $\mathcal P_t=\{t-p,\dots,t-1\}$ from stored past diagnostics
    \State Select $L_t$ by backward same-regime search over $L\in\{L_{\min},L_{\min}+h,\dots,L_{\max}\}$
    \Statex \hspace{\algorithmicindent} Accept the extension from $L-h$ to $L$ only if both $\mathcal C_t^{\mathrm{whole}}(L)$ and $\mathcal B_t^{\mathrm{inc}}(L)$ pass the same-regime test. Stop at the first rejected extension and set \(L_t\) to the last accepted length
    \State Form $\mathcal H_t(L_t)=\{e_s:s\in\mathcal C_t^{\mathrm{whole}}(L_t)\cup\mathcal P_t\}$
    \State Compute $r_t=Q_{1-\alpha_t}(\mathcal H_t(L_t))$
    \State Output
    \[
    \mathcal C_t(\alpha)=
    \left\{
    y\in\mathbb R^d:
    (y-\mu_t)^\top\widehat\Sigma_t^{+}(y-\mu_t)\le r_t^2
    \right\}
    \]
    \State Observe $y_t$ and compute
    \[
    e_t=\sqrt{(y_t-\mu_t)^\top\widehat\Sigma_t^{+}(y_t-\mu_t)}
    \]
    \State Compute $\tilde r_t=D^{-1}(y_t-\mu_t)$ and diagnostics
    \[
    s_{t,1}=e_t,\qquad
    s_{t,2}=\|\tilde r_t\|_\infty,\qquad
    s_{t,3}=|v_1^\top \tilde r_t|
    \]
    \State Store $e_t$ and $(s_{t,1},s_{t,2},s_{t,3})$
    \State Update $\alpha_t$ using ACI
\EndFor
\end{algorithmic}
\end{algorithm}

\section{Result Tables}
\label{app:appendix_result_table}

\subsection{Performance across forecasters}
\label{subsec:backbone_results}

In this subsection, we evaluate \textsc{Space} across heterogeneous probabilistic forecasters. Table~\ref{tab:backbone_summary} averages results over datasets for each forecaster, so it should be read as a robustness check on the intended application of the wrapper.

\begin{table}[t]
\captionsetup{aboveskip=\baselineskip, belowskip=\baselineskip}
\centering

\caption{Forecaster summary under \textsc{Space} at the $90\%$ target. Forecasters are GP-Copula \cite{salinas2019gpcopula}, TimeGrad \cite{rasul2021timegrad}, CSDI \cite{tashiro2021csdi}, TSFlow \cite{kollovieh2024tsflow}, DeepAR \cite{salinas2020deepar}, NsDiff \cite{ye2025nsdiff}, TMDM \cite{li2024tmdm}, and TACTiS-2 \cite{tactis2}. Coverage columns report average empirical coverage, while Gap columns report the average absolute coverage gaps. Gap reduction is reported as the relative percentage reduction $100 \times (\textsc{RawGap}-\textsc{SpaceGap})/\textsc{RawGap}$.}
\label{tab:backbone_summary}

\small
\begin{tabular}{lccccc}
\toprule
Backbone & Raw Cov & Raw Gap $\downarrow$ & SPACE Cov & SPACE Gap $\downarrow$ & Gap Red. \\
\midrule
CSDI & 0.787 & 0.125 & 0.902 & 0.004 & 96.8\% \\
DeepAR & 0.636 & 0.274 & 0.897 & 0.003 & 99.0\% \\
GP-Copula & 0.400 & 0.500 & 0.898 & 0.003 & 99.4\% \\
NsDiff & 0.824 & 0.189 & 0.899 & 0.003 & 98.4\% \\
TACTiS-2 & 0.376 & 0.524 & 0.903 & 0.004 & 99.3\% \\
TimeGrad & 0.371 & 0.529 & 0.897 & 0.005 & 99.0\% \\
TMDM & 0.644 & 0.308 & 0.899 & 0.002 & 99.4\% \\
TSFlow & 0.509 & 0.403 & 0.901 & 0.002 & 99.4\% \\
\bottomrule
\end{tabular}

\end{table}

\textsc{Space} improves calibration across a wide range of forecasters with very different raw uncertainty behavior. It substantially reduces the gap for clearly under-covered forecasters such as GP-Copula, TACTiS-2, and TimeGrad while also improving on raw forecasters that are already producing competitive results, such as NsDiff. This is consistent with the role of the wrapper: it is designed to pull realized coverage toward the nominal target regardless of whether the raw samples under- or over-cover.

\subsection{Full result tables}
We report the complete numerical results for all evaluated methods at the three target coverage levels in this subsection. These tables complement the aggregate summaries in the main text by providing the full dataset--forecaster level comparisons, including empirical joint coverage gap, and efficiency metrics for each conformal wrapper.

\label{app:full_result_tables}
\begingroup
\setlength{\LTcapwidth}{\textwidth}
\tiny
\setlength{\tabcolsep}{2.0pt}
\renewcommand{\arraystretch}{0.88}
\setlength{\LTleft}{0pt plus 1fill}
\setlength{\LTright}{0pt plus 1fill}
\begin{longtable}{llcccccc|cccccc}
\caption{Detailed conformal wrapper comparison at the $90\%$ target. The table reports coverage gap and log-volume for each dataset--backbone pair across all conformal wrappers. Lower is better for both metrics.}
\label{tab:wrapper_detailed_90}\\
\toprule
& & \multicolumn{6}{c}{Coverage gap $\downarrow$} & \multicolumn{6}{c}{LogVol $\downarrow$} \\
\cmidrule(lr){3-8} \cmidrule(lr){9-14}
Dataset & Backbone
& SPACE & \shortstack{MultiDim\\SPCI} & \shortstack{Local\\Ell.} & \shortstack{Emp.\\Cop.} & CopulaCPTS & \shortstack{Gauss.\\Cop.}
& SPACE & \shortstack{MultiDim\\SPCI} & \shortstack{Local\\Ell.} & \shortstack{Emp.\\Cop.} & CopulaCPTS & \shortstack{Gauss.\\Cop.} \\
\midrule
\endfirsthead

\multicolumn{14}{c}{\tablename~\thetable{} continued from previous page}\\
\toprule
& & \multicolumn{6}{c}{Coverage gap $\downarrow$} & \multicolumn{6}{c}{LogVol $\downarrow$} \\
\cmidrule(lr){3-8} \cmidrule(lr){9-14}
Dataset & Backbone
& SPACE & \shortstack{MultiDim\\SPCI} & \shortstack{Local\\Ell.} & \shortstack{Emp.\\Cop.} & CopulaCPTS & \shortstack{Gauss.\\Cop.}
& SPACE & \shortstack{MultiDim\\SPCI} & \shortstack{Local\\Ell.} & \shortstack{Emp.\\Cop.} & CopulaCPTS & \shortstack{Gauss.\\Cop.} \\
\midrule
\endhead

\midrule
\multicolumn{14}{r}{Continued on next page}\\
\endfoot

\bottomrule
\endlastfoot
\multirow{7}{*}{electricity} & DeepAR & 0.004 & 0.139 & 0.139 & 0.121 & 0.083 & \textbf{0.002} & \textbf{1797.94} & 1743.91 & 1751.16 & 2072.63 & 2816.26 & 2166.29 \\
 & GP-Copula & \textbf{0.002} & 0.142 & 0.151 & 0.254 & 0.072 & 0.030 & \textbf{1827.45} & 1851.09 & 1842.54 & 2082.68 & 2970.49 & 2186.00 \\
 & NsDiff & \textbf{0.000} & 0.079 & 0.426 & 0.585 & 0.099 & 0.124 & \textbf{3436.61} & 1916.57 & 1879.74 & 2046.48 & 3047.53 & 2201.14 \\
 & TACTiS-2 & \textbf{0.002} & 0.134 & 0.183 & 0.080 & 0.081 & 0.028 & \textbf{1867.36} & 1790.10 & 1791.68 & 2134.99 & 2848.78 & 2187.38 \\
 & TimeGrad & \textbf{0.003} & 0.073 & 0.384 & 0.580 & 0.075 & 0.299 & \textbf{3955.57} & 2066.06 & 2034.33 & 2126.94 & 3086.63 & 2236.49 \\
 & TMDM & \textbf{0.002} & 0.065 & 0.390 & 0.477 & 0.043 & 0.109 & \textbf{3446.15} & 1986.95 & 1956.58 & 2102.96 & 3140.46 & 2221.45 \\
 & TSFlow & 0.003 & 0.117 & 0.121 & 0.142 & 0.077 & \textbf{0.002} & 2787.87 & 1758.87 & 1756.52 & 2051.49 & 2803.15 & \textbf{2191.19} \\
\midrule
\multirow{8}{*}{etth1} & CSDI & \textbf{0.003} & 0.007 & 0.054 & 0.024 & 0.099 & 0.027 & 10.52 & \textbf{7.53} & 7.51 & 11.49 & 20.79 & 11.69 \\
 & DeepAR & \textbf{0.002} & 0.015 & 0.041 & 0.028 & 0.058 & 0.041 & 9.79 & \textbf{7.76} & 7.86 & 11.13 & 15.68 & 11.56 \\
 & GP-Copula & \textbf{0.005} & 0.011 & 0.058 & 0.047 & 0.037 & 0.053 & 10.26 & \textbf{8.70} & 9.17 & 12.08 & 16.50 & 12.40 \\
 & NsDiff & \textbf{0.005} & 0.014 & 0.063 & 0.023 & 0.034 & 0.023 & 11.90 & \textbf{9.05} & 9.49 & 11.81 & 16.47 & 11.81 \\
 & TACTiS-2 & \textbf{0.002} & 0.005 & 0.064 & 0.027 & 0.034 & 0.034 & 10.30 & \textbf{7.66} & 8.18 & 11.49 & 15.11 & 11.67 \\
 & TimeGrad & \textbf{0.008} & 0.025 & 0.165 & 0.033 & 0.058 & 0.009 & 16.19 & 11.96 & 13.22 & 14.35 & 19.49 & \textbf{14.55} \\
 & TMDM & \textbf{0.000} & 0.002 & 0.057 & 0.031 & 0.060 & 0.043 & 11.46 & \textbf{7.67} & 8.13 & 11.81 & 16.61 & 12.16 \\
 & TSFlow & \textbf{0.003} & 0.012 & 0.063 & 0.030 & 0.100 & 0.037 & 10.37 & \textbf{7.87} & 8.14 & 11.19 & 20.87 & 11.63 \\
\midrule
\multirow{8}{*}{ettm1} & CSDI & \textbf{0.000} & 0.017 & 0.054 & 0.030 & 0.009 & 0.042 & 6.59 & \textbf{4.60} & 5.27 & 9.06 & 11.78 & 9.46 \\
 & DeepAR & \textbf{0.001} & 0.009 & 0.059 & 0.028 & 0.019 & 0.038 & 6.74 & \textbf{4.82} & 5.74 & 8.80 & 11.99 & 9.19 \\
 & GP-Copula & \textbf{0.000} & 0.005 & 0.049 & 0.021 & 0.006 & 0.038 & \textbf{5.53} & 5.95 & 6.59 & 9.74 & 12.44 & 10.21 \\
 & NsDiff & 0.005 & \textbf{0.002} & 0.046 & 0.016 & 0.009 & 0.029 & 9.04 & \textbf{7.28} & 7.57 & 9.68 & 13.78 & 10.00 \\
 & TACTiS-2 & \textbf{0.001} & 0.004 & 0.054 & 0.027 & \textbf{0.001} & 0.037 & 6.55 & \textbf{4.91} & 5.47 & 9.14 & 11.84 & 9.55 \\
 & TimeGrad & 0.011 & 0.021 & 0.117 & 0.015 & 0.076 & \textbf{0.007} & 13.51 & 9.18 & 10.30 & \textbf{12.40} & 18.36 & 12.55 \\
 & TMDM & \textbf{0.001} & 0.005 & 0.053 & 0.026 & 0.003 & 0.038 & 7.26 & \textbf{4.70} & 5.33 & 8.94 & 11.50 & 9.31 \\
 & TSFlow & \textbf{0.000} & 0.018 & 0.057 & 0.025 & 0.011 & 0.037 & 7.22 & \textbf{5.44} & 5.94 & 9.16 & 11.69 & 9.53 \\
\midrule
\multirow{8}{*}{ettm2} & CSDI & \textbf{0.000} & 0.009 & 0.019 & 0.004 & 0.004 & 0.019 & 10.59 & 6.82 & \textbf{6.33} & 10.02 & 12.37 & 10.33 \\
 & DeepAR & \textbf{0.000} & 0.006 & 0.036 & 0.004 & 0.007 & 0.010 & 11.17 & \textbf{7.50} & 7.68 & 10.43 & 12.61 & 10.64 \\
 & GP-Copula & \textbf{0.000} & 0.007 & 0.054 & 0.019 & 0.019 & 0.008 & 9.72 & \textbf{8.26} & 7.73 & 10.87 & 12.22 & 11.09 \\
 & NsDiff & \textbf{0.001} & \textbf{0.001} & 0.009 & 0.010 & 0.009 & 0.012 & 12.61 & 10.50 & \textbf{10.40} & 12.30 & 13.55 & 12.66 \\
 & TACTiS-2 & \textbf{0.000} & 0.010 & 0.009 & \textbf{0.000} & 0.009 & 0.017 & 10.84 & 6.65 & \textbf{6.33} & 9.57 & 12.52 & 9.85 \\
 & TimeGrad & \textbf{0.000} & 0.010 & 0.039 & 0.011 & 0.074 & 0.019 & 14.24 & \textbf{13.97} & 14.12 & 15.15 & 18.35 & 15.25 \\
 & TMDM & \textbf{0.000} & 0.004 & 0.012 & 0.004 & 0.005 & 0.015 & 10.76 & 7.23 & \textbf{6.91} & 10.04 & 12.24 & 10.28 \\
 & TSFlow & \textbf{0.000} & 0.005 & 0.009 & 0.001 & 0.026 & 0.007 & 19.02 & \textbf{15.57} & 16.04 & 16.59 & 19.81 & 16.55 \\
\midrule
\multirow{8}{*}{exchange} & CSDI & 0.014 & 0.018 & \textbf{0.001} & 0.041 & 0.044 & 0.041 & \textbf{-36.70} & -23.14 & -35.28 & -36.88 & -29.04 & -36.88 \\
 & DeepAR & \textbf{0.009} & 0.055 & 0.018 & 0.140 & 0.024 & 0.094 & \textbf{-29.75} & -18.29 & -28.40 & -30.11 & -22.54 & -29.78 \\
 & GP-Copula & 0.012 & 0.025 & \textbf{0.002} & 0.055 & 0.014 & 0.055 & \textbf{-34.64} & -20.83 & -31.97 & -32.78 & -27.08 & -32.78 \\
 & NsDiff & \textbf{0.005} & 0.048 & 0.044 & 0.134 & 0.051 & 0.134 & \textbf{-29.63} & -21.82 & -26.09 & -29.38 & -21.91 & -29.38 \\
 & TACTiS-2 & \textbf{0.014} & 0.018 & 0.024 & 0.048 & 0.064 & 0.048 & \textbf{-36.41} & -23.43 & -34.81 & -36.66 & -28.35 & -36.66 \\
 & TimeGrad & \textbf{0.009} & 0.045 & 0.551 & 0.818 & 0.821 & 0.818 & \textbf{-11.24} & -9.03 & -12.69 & -17.86 & -6.30 & -17.86 \\
 & TMDM & \textbf{0.009} & 0.012 & 0.047 & 0.439 & 0.196 & 0.439 & \textbf{-13.05} & -9.87 & -15.25 & -18.56 & -9.52 & -18.56 \\
 & TSFlow & \textbf{0.005} & 0.028 & 0.166 & 0.117 & 0.101 & 0.084 & \textbf{-33.11} & -20.75 & -33.39 & -35.09 & -27.04 & -34.92 \\
\midrule
\multirow{7}{*}{traffic} & DeepAR & \textbf{0.001} & 0.070 & 0.034 & 0.367 & 0.069 & 0.367 & \textbf{-2004.30} & -1655.92 & -1917.09 & -1110.90 & -1444.22 & -1110.90 \\
 & GP-Copula & \textbf{0.003} & 0.060 & 0.018 & 0.417 & 0.057 & 0.417 & \textbf{-2043.70} & -1485.45 & -1843.90 & -1047.90 & -1498.27 & -1047.90 \\
 & NsDiff & \textbf{0.004} & 0.007 & 0.058 & 0.503 & 0.027 & 0.503 & \textbf{-1666.08} & -1049.60 & -1774.45 & -1090.52 & -1513.42 & -1090.52 \\
 & TACTiS-2 & \textbf{0.006} & 0.054 & 0.023 & 0.419 & 0.057 & 0.419 & \textbf{-1947.96} & -1612.59 & -1837.45 & -1075.69 & -1471.75 & -1075.69 \\
 & TimeGrad & \textbf{0.006} & 0.070 & 0.078 & 0.659 & 0.039 & 0.659 & \textbf{-1419.68} & -1282.09 & -1501.35 & -955.67 & -1165.37 & -955.67 \\
 & TMDM & \textbf{0.000} & 0.020 & 0.035 & 0.436 & 0.015 & 0.436 & \textbf{-1524.18} & -704.16 & -1708.73 & -1086.10 & -1484.19 & -1086.10 \\
 & TSFlow & \textbf{0.005} & 0.041 & 0.055 & 0.421 & 0.059 & 0.421 & \textbf{-2008.51} & -1600.29 & -1888.35 & -1111.23 & -1576.92 & -1111.23 \\
\midrule
\multirow{8}{*}{weather} & CSDI & \textbf{0.002} & 0.007 & 0.059 & 0.115 & 0.021 & 0.045 & 41.87 & \textbf{5.26} & 2.86 & 25.90 & 100.70 & 28.92 \\
 & DeepAR & \textbf{0.003} & 0.004 & 0.036 & 0.157 & 0.027 & 0.125 & 54.96 & \textbf{35.45} & 39.73 & 49.29 & 113.17 & 50.42 \\
 & GP-Copula & \textbf{0.001} & 0.009 & 0.057 & 0.107 & 0.005 & 0.049 & 43.25 & \textbf{22.63} & 21.91 & 41.69 & 102.62 & 44.55 \\
 & NsDiff & \textbf{0.001} & 0.003 & 0.020 & 0.096 & 0.024 & 0.058 & 61.59 & \textbf{53.72} & 55.02 & 59.86 & 108.53 & 61.17 \\
 & TACTiS-2 & 0.001 & \textbf{0.000} & 0.020 & 0.051 & 0.020 & 0.045 & 43.82 & \textbf{8.22} & 4.96 & 28.99 & 100.33 & 29.32 \\
 & TimeGrad & \textbf{0.001} & 0.001 & 0.085 & 0.271 & 0.010 & 0.229 & 57.09 & \textbf{41.14} & 42.66 & 46.96 & 106.63 & 48.13 \\
 & TMDM & \textbf{0.001} & 0.004 & \textbf{0.001} & 0.127 & 0.022 & 0.068 & 50.32 & \textbf{35.10} & 36.39 & 43.23 & 102.61 & 45.18 \\
 & TSFlow & \textbf{0.000} & 0.170 & 0.034 & 0.264 & 0.026 & 0.221 & \textbf{60.52} & 67.85 & 79.83 & 43.48 & 104.69 & 44.51 \\
\end{longtable}

\endgroup

\begingroup
\setlength{\LTcapwidth}{\textwidth}
\tiny
\setlength{\tabcolsep}{2.0pt}
\renewcommand{\arraystretch}{0.88}
\setlength{\LTleft}{0pt plus 1fill}
\setlength{\LTright}{0pt plus 1fill}
\begin{longtable}{llcccccc|cccccc}
\caption{Detailed conformal wrapper comparison at the $50\%$ target. The table reports coverage gap and log-volume for each dataset--backbone pair across all conformal wrappers. Lower is better for both metrics.}
\label{tab:wrapper_detailed_50}\\
\toprule
& & \multicolumn{6}{c}{Coverage gap $\downarrow$} & \multicolumn{6}{c}{LogVol $\downarrow$} \\
\cmidrule(lr){3-8} \cmidrule(lr){9-14}
Dataset & Backbone
& SPACE & \shortstack{MultiDim\\SPCI} & \shortstack{Local\\Ell.} & \shortstack{Emp.\\Cop.} & CopulaCPTS & \shortstack{Gauss.\\Cop.}
& SPACE & \shortstack{MultiDim\\SPCI} & \shortstack{Local\\Ell.} & \shortstack{Emp.\\Cop.} & CopulaCPTS & \shortstack{Gauss.\\Cop.} \\
\midrule
\endfirsthead

\multicolumn{14}{c}{\tablename~\thetable{} continued from previous page}\\
\toprule
& & \multicolumn{6}{c}{Coverage gap $\downarrow$} & \multicolumn{6}{c}{LogVol $\downarrow$} \\
\cmidrule(lr){3-8} \cmidrule(lr){9-14}
Dataset & Backbone
& SPACE & \shortstack{MultiDim\\SPCI} & \shortstack{Local\\Ell.} & \shortstack{Emp.\\Cop.} & CopulaCPTS & \shortstack{Gauss.\\Cop.}
& SPACE & \shortstack{MultiDim\\SPCI} & \shortstack{Local\\Ell.} & \shortstack{Emp.\\Cop.} & CopulaCPTS & \shortstack{Gauss.\\Cop.} \\
\midrule
\endhead

\midrule
\multicolumn{14}{r}{Continued on next page}\\
\endfoot

\bottomrule
\endlastfoot
\multirow{7}{*}{electricity} & DeepAR & \textbf{0.004} & 0.111 & 0.122 & 0.197 & 0.458 & 0.121 & \textbf{1716.61} & 1696.69 & 1684.79 & 1884.26 & 2702.49 & 1914.07 \\
 & GP-Copula & \textbf{0.001} & 0.100 & 0.158 & 0.158 & 0.226 & 0.088 & \textbf{1765.49} & 1802.28 & 1766.75 & 1990.40 & 2564.03 & 2017.31 \\
 & NsDiff & \textbf{0.002} & 0.115 & 0.305 & 0.385 & 0.378 & 0.287 & \textbf{2694.55} & 1876.96 & 1834.83 & 1972.64 & 2799.25 & 2010.21 \\
 & TACTiS-2 & \textbf{0.005} & 0.106 & 0.179 & 0.140 & 0.159 & 0.007 & \textbf{1776.44} & 1743.45 & 1725.51 & 1964.06 & 2500.17 & 1993.67 \\
 & TimeGrad & \textbf{0.001} & 0.108 & 0.232 & 0.373 & 0.175 & 0.337 & \textbf{3304.03} & 2042.12 & 1995.94 & 2076.98 & 2819.30 & 2088.17 \\
 & TMDM & \textbf{0.002} & 0.106 & 0.262 & 0.331 & 0.095 & 0.195 & \textbf{2759.85} & 1960.32 & 1928.50 & 2040.14 & 2900.43 & 2071.59 \\
 & TSFlow & \textbf{0.001} & 0.105 & 0.103 & 0.193 & 0.424 & 0.098 & \textbf{2351.84} & 1714.78 & 1695.45 & 1907.09 & 2643.04 & 1933.06 \\
\midrule
\multirow{8}{*}{etth1} & CSDI & \textbf{0.001} & 0.012 & 0.160 & 0.065 & 0.499 & 0.105 & 4.63 & \textbf{4.16} & 3.79 & 5.78 & 20.79 & 6.30 \\
 & DeepAR & \textbf{0.005} & 0.038 & 0.216 & 0.081 & 0.458 & 0.147 & \textbf{4.39} & 4.33 & 4.43 & 5.69 & 15.68 & 6.43 \\
 & GP-Copula & \textbf{0.008} & 0.027 & 0.189 & 0.095 & 0.047 & 0.144 & \textbf{5.33} & 5.58 & 5.54 & 7.15 & 8.72 & 7.77 \\
 & NsDiff & \textbf{0.004} & 0.011 & 0.174 & 0.077 & 0.199 & 0.115 & 6.18 & \textbf{6.02} & 6.06 & 6.95 & 12.25 & 7.45 \\
 & TACTiS-2 & 0.005 & \textbf{0.001} & 0.183 & 0.113 & 0.094 & 0.156 & 4.68 & \textbf{4.18} & 4.18 & 6.26 & 8.28 & 6.72 \\
 & TimeGrad & \textbf{0.002} & 0.038 & 0.196 & 0.110 & 0.160 & 0.027 & \textbf{13.35} & 10.41 & 10.36 & 10.60 & 16.19 & 11.24 \\
 & TMDM & 0.008 & \textbf{0.006} & 0.128 & 0.060 & 0.065 & 0.115 & 6.05 & \textbf{4.85} & 4.77 & 6.77 & 9.49 & 7.26 \\
 & TSFlow & 0.006 & \textbf{0.001} & 0.183 & 0.081 & 0.500 & 0.140 & 4.70 & \textbf{4.44} & 4.33 & 5.98 & 20.87 & 6.58 \\
\midrule
\multirow{8}{*}{ettm1} & CSDI & \textbf{0.001} & 0.018 & 0.152 & 0.078 & 0.037 & 0.116 & 1.37 & \textbf{1.05} & 1.07 & 2.25 & 3.66 & 2.64 \\
 & DeepAR & \textbf{0.001} & 0.015 & 0.175 & 0.100 & 0.057 & 0.135 & 1.48 & \textbf{1.09} & 1.49 & 2.42 & 3.67 & 2.82 \\
 & GP-Copula & \textbf{0.001} & 0.020 & 0.152 & 0.061 & 0.047 & 0.091 & \textbf{1.04} & 2.51 & 2.22 & 3.26 & 4.51 & 3.64 \\
 & NsDiff & \textbf{0.000} & 0.004 & 0.110 & 0.016 & 0.004 & 0.074 & \textbf{4.41} & 4.71 & 4.35 & 4.85 & 7.80 & 5.30 \\
 & TACTiS-2 & 0.000 & \textbf{0.000} & 0.150 & 0.062 & 0.040 & 0.119 & 1.31 & \textbf{1.31} & 1.30 & 2.26 & 3.78 & 2.82 \\
 & TimeGrad & \textbf{0.003} & 0.063 & 0.204 & 0.113 & 0.240 & 0.043 & \textbf{10.58} & 7.31 & 6.70 & 8.15 & 14.17 & 8.80 \\
 & TMDM & \textbf{0.000} & 0.019 & 0.188 & 0.074 & 0.022 & 0.129 & 1.87 & \textbf{1.29} & 1.65 & 2.45 & 3.68 & 3.01 \\
 & TSFlow & \textbf{0.001} & 0.033 & 0.145 & 0.067 & 0.050 & 0.107 & \textbf{2.18} & 1.91 & 2.11 & 2.89 & 4.18 & 3.27 \\
\midrule
\multirow{8}{*}{ettm2} & CSDI & \textbf{0.001} & 0.002 & 0.010 & 0.027 & 0.049 & 0.069 & 5.15 & 4.03 & \textbf{2.89} & 5.45 & 6.89 & 5.98 \\
 & DeepAR & \textbf{0.002} & 0.008 & 0.062 & \textbf{0.002} & 0.023 & 0.039 & 6.07 & \textbf{5.62} & 4.46 & 6.61 & 7.16 & 7.03 \\
 & GP-Copula & 0.004 & 0.009 & 0.042 & 0.003 & \textbf{0.000} & 0.039 & \textbf{4.97} & 5.18 & 3.86 & 6.71 & 7.17 & 7.21 \\
 & NsDiff & \textbf{0.000} & 0.004 & 0.055 & 0.029 & 0.006 & 0.025 & 8.95 & \textbf{8.84} & 7.66 & 8.99 & 9.82 & 9.34 \\
 & TACTiS-2 & 0.001 & \textbf{0.001} & 0.007 & 0.013 & 0.047 & 0.069 & 4.64 & 4.19 & \textbf{2.62} & 4.16 & 6.81 & 5.09 \\
 & TimeGrad & \textbf{0.003} & 0.020 & 0.120 & 0.062 & 0.201 & 0.081 & \textbf{12.10} & 12.54 & 12.20 & 12.41 & 14.85 & 12.55 \\
 & TMDM & \textbf{0.002} & 0.005 & 0.015 & 0.019 & 0.028 & 0.047 & 6.09 & 5.13 & \textbf{4.01} & 6.17 & 6.92 & 6.41 \\
 & TSFlow & 0.002 & 0.003 & 0.041 & \textbf{0.001} & 0.024 & 0.016 & 13.64 & 14.05 & 13.38 & \textbf{13.63} & 16.48 & 13.75 \\
\midrule
\multirow{8}{*}{exchange} & CSDI & \textbf{0.000} & 0.026 & 0.036 & 0.128 & 0.326 & 0.102 & \textbf{-41.58} & -25.69 & -39.72 & -42.43 & -31.77 & -42.04 \\
 & DeepAR & 0.016 & 0.023 & 0.082 & 0.049 & 0.411 & \textbf{0.003} & \textbf{-33.27} & -19.66 & -32.02 & -33.12 & -23.03 & -32.77 \\
 & GP-Copula & 0.010 & 0.010 & 0.122 & \textbf{0.003} & 0.345 & 0.023 & \textbf{-39.91} & -23.05 & -36.72 & -38.24 & -29.11 & -37.83 \\
 & NsDiff & \textbf{0.010} & 0.026 & 0.013 & 0.086 & 0.434 & 0.013 & \textbf{-32.50} & -23.40 & -30.04 & -32.25 & -23.10 & -31.88 \\
 & TACTiS-2 & \textbf{0.016} & 0.046 & 0.118 & 0.086 & 0.464 & 0.056 & \textbf{-42.47} & -25.88 & -39.44 & -42.62 & -28.35 & -42.14 \\
 & TimeGrad & \textbf{0.003} & 0.066 & 0.490 & 0.500 & 0.480 & 0.500 & \textbf{-12.22} & -9.98 & -14.18 & -19.46 & -6.49 & -19.43 \\
 & TMDM & \textbf{0.023} & \textbf{0.023} & 0.046 & 0.230 & 0.204 & 0.227 & -13.76 & -10.77 & -16.31 & -19.98 & -9.52 & -19.94 \\
 & TSFlow & \textbf{0.007} & 0.049 & 0.283 & 0.266 & 0.214 & 0.161 & \textbf{-36.90} & -22.61 & -37.73 & -39.89 & -27.48 & -38.94 \\
\midrule
\multirow{7}{*}{traffic} & DeepAR & \textbf{0.001} & 0.078 & 0.094 & 0.205 & 0.111 & 0.174 & \textbf{-2751.54} & -2041.51 & -2780.99 & -1496.57 & -2275.99 & -1413.46 \\
 & GP-Copula & \textbf{0.001} & 0.090 & 0.094 & 0.226 & 0.209 & 0.188 & \textbf{-2763.05} & -1871.26 & -2697.83 & -1431.22 & -2037.54 & -1349.32 \\
 & NsDiff & \textbf{0.001} & 0.053 & 0.117 & 0.359 & 0.117 & 0.103 & \textbf{-2252.30} & -1306.73 & -2347.17 & -1488.17 & -2020.75 & -1090.52 \\
 & TACTiS-2 & \textbf{0.003} & 0.090 & 0.110 & 0.232 & 0.115 & 0.187 & \textbf{-2633.07} & -2013.00 & -2766.84 & -1456.24 & -2197.27 & -1373.58 \\
 & TimeGrad & \textbf{0.006} & 0.093 & 0.127 & 0.369 & 0.323 & 0.356 & \textbf{-1977.18} & -1547.09 & -2157.70 & -1253.83 & -1439.78 & -1193.78 \\
 & TMDM & \textbf{0.000} & 0.001 & 0.075 & 0.315 & 0.321 & 0.036 & \textbf{-2036.92} & -905.44 & -2123.10 & -1467.98 & -1649.90 & -1086.10 \\
 & TSFlow & \textbf{0.011} & 0.067 & 0.103 & 0.271 & 0.125 & 0.238 & \textbf{-2693.86} & -1980.92 & -2595.17 & -1519.05 & -2249.78 & -1434.20 \\
\midrule
\multirow{8}{*}{weather} & CSDI & \textbf{0.004} & 0.021 & 0.178 & 0.144 & 0.074 & 0.044 & \textbf{24.26} & -2.85 & -10.86 & -3.75 & 53.21 & 5.73 \\
 & DeepAR & \textbf{0.004} & 0.014 & 0.086 & 0.179 & 0.045 & 0.114 & 37.54 & \textbf{30.15} & 30.30 & 32.18 & 67.44 & 35.49 \\
 & GP-Copula & \textbf{0.004} & 0.004 & 0.114 & 0.139 & 0.074 & 0.052 & 26.74 & \textbf{13.03} & 7.31 & -34.41 & 58.39 & 0.05 \\
 & NsDiff & \textbf{0.004} & 0.009 & 0.039 & 0.081 & 0.367 & 0.015 & 49.90 & \textbf{49.44} & 48.36 & 50.91 & 104.03 & 52.31 \\
 & TACTiS-2 & 0.002 & \textbf{0.001} & 0.155 & 0.158 & 0.070 & 0.074 & 23.38 & \textbf{-0.54} & -8.53 & 0.08 & 53.40 & 6.62 \\
 & TimeGrad & \textbf{0.004} & 0.019 & 0.016 & 0.143 & 0.035 & 0.103 & 41.73 & 37.48 & \textbf{35.29} & 35.94 & 81.61 & 37.35 \\
 & TMDM & \textbf{0.002} & 0.040 & 0.014 & 0.133 & 0.072 & 0.057 & 34.83 & 30.26 & \textbf{30.21} & 32.24 & 77.09 & 34.27 \\
 & TSFlow & \textbf{0.001} & 0.090 & 0.031 & 0.284 & 0.055 & 0.225 & \textbf{47.89} & 66.87 & 76.82 & 25.99 & 71.28 & 29.01 \\
\end{longtable}

\endgroup

\begingroup
\setlength{\LTcapwidth}{\textwidth}
\tiny
\setlength{\tabcolsep}{2.0pt}
\renewcommand{\arraystretch}{0.88}
\setlength{\LTleft}{0pt plus 1fill}
\setlength{\LTright}{0pt plus 1fill}
\begin{longtable}{llcccccc|cccccc}
\caption{Detailed conformal wrapper comparison at the $95\%$ target. The table reports coverage gap and log-volume for each dataset--backbone pair across all conformal wrappers. Lower is better for both metrics.}
\label{tab:wrapper_detailed_95}\\
\toprule
& & \multicolumn{6}{c}{Coverage gap $\downarrow$} & \multicolumn{6}{c}{LogVol $\downarrow$} \\
\cmidrule(lr){3-8} \cmidrule(lr){9-14}
Dataset & Backbone
& SPACE & \shortstack{MultiDim\\SPCI} & \shortstack{Local\\Ell.} & \shortstack{Emp.\\Cop.} & CopulaCPTS & \shortstack{Gauss.\\Cop.}
& SPACE & \shortstack{MultiDim\\SPCI} & \shortstack{Local\\Ell.} & \shortstack{Emp.\\Cop.} & CopulaCPTS & \shortstack{Gauss.\\Cop.} \\
\midrule
\endfirsthead

\multicolumn{14}{c}{\tablename~\thetable{} continued from previous page}\\
\toprule
& & \multicolumn{6}{c}{Coverage gap $\downarrow$} & \multicolumn{6}{c}{LogVol $\downarrow$} \\
\cmidrule(lr){3-8} \cmidrule(lr){9-14}
Dataset & Backbone
& SPACE & \shortstack{MultiDim\\SPCI} & \shortstack{Local\\Ell.} & \shortstack{Emp.\\Cop.} & CopulaCPTS & \shortstack{Gauss.\\Cop.}
& SPACE & \shortstack{MultiDim\\SPCI} & \shortstack{Local\\Ell.} & \shortstack{Emp.\\Cop.} & CopulaCPTS & \shortstack{Gauss.\\Cop.} \\
\midrule
\endhead

\midrule
\multicolumn{14}{r}{Continued on next page}\\
\endfoot

\bottomrule
\endlastfoot
\multirow{7}{*}{electricity} & DeepAR & \textbf{0.004} & 0.111 & 0.068 & 0.171 & 0.042 & 0.052 & \textbf{1823.07} & 1759.72 & 1790.24 & 2072.63 & 2949.63 & 2166.29 \\
 & GP-Copula & \textbf{0.000} & 0.116 & 0.098 & 0.304 & 0.040 & 0.080 & \textbf{1854.07} & 1866.89 & 1864.68 & 2082.68 & 3051.98 & 2186.00 \\
 & NsDiff & \textbf{0.000} & 0.057 & 0.369 & 0.635 & 0.050 & 0.174 & \textbf{3683.00} & 1928.18 & 1895.89 & 2046.48 & 3144.25 & 2201.14 \\
 & TACTiS-2 & \textbf{0.003} & 0.121 & 0.110 & 0.130 & 0.042 & 0.078 & \textbf{1903.58} & 1805.75 & 1821.25 & 2134.99 & 2957.35 & 2187.38 \\
 & TimeGrad & \textbf{0.003} & 0.067 & 0.402 & 0.630 & 0.044 & 0.349 & \textbf{4202.56} & 2073.73 & 2041.09 & 2126.94 & 3188.55 & 2236.49 \\
 & TMDM & \textbf{0.002} & 0.065 & 0.376 & 0.527 & 0.037 & 0.159 & \textbf{3757.88} & 1997.86 & 1963.54 & 2102.96 & 3220.00 & 2221.45 \\
 & TSFlow & \textbf{0.001} & 0.110 & 0.085 & 0.192 & 0.040 & 0.052 & \textbf{2967.57} & 1777.25 & 1781.26 & 2051.49 & 2915.85 & 2191.19 \\
\midrule
\multirow{8}{*}{etth1} & CSDI & 0.003 & \textbf{0.000} & 0.034 & 0.011 & 0.049 & 0.021 & 12.96 & \textbf{8.25} & 9.33 & 12.86 & 20.79 & 13.36 \\
 & DeepAR & \textbf{0.000} & 0.008 & 0.024 & 0.011 & 0.030 & 0.024 & 11.97 & \textbf{8.60} & 9.25 & 12.45 & 17.49 & 13.10 \\
 & GP-Copula & 0.005 & \textbf{0.000} & 0.030 & 0.026 & 0.024 & 0.031 & 12.31 & \textbf{9.45} & 10.60 & 13.08 & 18.69 & 13.46 \\
 & NsDiff & \textbf{0.000} & 0.010 & 0.027 & 0.014 & 0.030 & 0.024 & 13.53 & \textbf{9.95} & 10.56 & 12.99 & 18.42 & 13.68 \\
 & TACTiS-2 & \textbf{0.000} & 0.003 & 0.033 & 0.018 & 0.028 & 0.026 & 12.75 & \textbf{8.58} & 9.64 & 12.75 & 17.76 & 13.28 \\
 & TimeGrad & \textbf{0.003} & 0.010 & 0.109 & 0.038 & 0.010 & 0.029 & 16.81 & \textbf{12.42} & 13.86 & 14.95 & 20.49 & 15.13 \\
 & TMDM & 0.002 & \textbf{0.001} & 0.028 & 0.013 & 0.028 & 0.018 & 13.72 & \textbf{8.43} & 9.43 & 12.84 & 18.38 & 13.45 \\
 & TSFlow & \textbf{0.000} & \textbf{0.000} & 0.030 & 0.020 & 0.050 & 0.024 & 13.29 & \textbf{8.77} & 9.61 & 12.57 & 20.87 & 12.97 \\
\midrule
\multirow{8}{*}{ettm1} & CSDI & \textbf{0.000} & 0.002 & 0.027 & 0.013 & 0.007 & 0.024 & 9.32 & \textbf{5.91} & 6.57 & 10.49 & 14.63 & 11.21 \\
 & DeepAR & \textbf{0.000} & 0.002 & 0.028 & 0.009 & 0.002 & 0.022 & 9.18 & \textbf{5.92} & 7.19 & 10.18 & 14.06 & 10.97 \\
 & GP-Copula & \textbf{0.001} & 0.003 & 0.028 & 0.024 & 0.005 & 0.024 & 7.55 & \textbf{7.00} & 8.26 & 11.77 & 14.89 & 11.77 \\
 & NsDiff & 0.007 & \textbf{0.003} & 0.023 & 0.009 & 0.006 & 0.013 & 11.44 & \textbf{8.33} & 8.43 & 11.27 & 15.32 & 11.47 \\
 & TACTiS-2 & \textbf{0.000} & 0.006 & 0.026 & 0.009 & 0.001 & 0.019 & 9.45 & \textbf{6.18} & 7.07 & 10.58 & 14.39 & 11.32 \\
 & TimeGrad & 0.019 & \textbf{0.004} & 0.067 & 0.026 & 0.038 & 0.016 & 14.26 & \textbf{9.89} & 10.95 & 13.45 & 19.58 & 13.76 \\
 & TMDM & \textbf{0.000} & 0.003 & 0.025 & 0.008 & 0.002 & 0.023 & 9.51 & \textbf{5.73} & 6.72 & 10.31 & 14.23 & 11.10 \\
 & TSFlow & 0.000 & 0.004 & 0.032 & 0.013 & \textbf{0.000} & 0.025 & 9.55 & \textbf{6.53} & 7.45 & 10.49 & 14.21 & 11.31 \\
\midrule
\multirow{8}{*}{ettm2} & CSDI & \textbf{0.001} & 0.005 & 0.016 & 0.007 & 0.003 & 0.008 & 11.95 & 7.65 & \textbf{7.36} & 10.87 & 13.51 & 11.34 \\
 & DeepAR & \textbf{0.000} & 0.001 & 0.027 & 0.011 & 0.003 & 0.003 & 12.64 & \textbf{8.16} & 8.62 & 11.33 & 13.79 & 11.75 \\
 & GP-Copula & \textbf{0.000} & 0.002 & 0.029 & 0.016 & 0.011 & 0.006 & 11.36 & \textbf{8.83} & 8.88 & 11.97 & 13.67 & 12.87 \\
 & NsDiff & 0.002 & 0.005 & 0.005 & \textbf{0.002} & 0.011 & 0.004 & 13.81 & \textbf{11.13} & 11.17 & 13.53 & 14.49 & 13.59 \\
 & TACTiS-2 & \textbf{0.001} & 0.009 & 0.020 & 0.004 & 0.011 & 0.009 & 12.34 & 7.39 & \textbf{7.26} & 10.87 & 13.79 & 11.27 \\
 & TimeGrad & 0.002 & \textbf{0.001} & 0.024 & 0.013 & 0.037 & 0.015 & 14.95 & \textbf{14.38} & 14.84 & 16.01 & 19.11 & 16.09 \\
 & TMDM & \textbf{0.000} & 0.002 & 0.011 & 0.009 & 0.002 & 0.004 & 12.25 & 7.96 & \textbf{7.72} & 10.89 & 13.42 & 11.29 \\
 & TSFlow & \textbf{0.001} & 0.002 & 0.002 & 0.008 & 0.004 & 0.005 & 20.47 & \textbf{16.18} & 16.64 & 17.26 & 20.82 & 17.31 \\
\midrule
\multirow{8}{*}{exchange} & CSDI & 0.004 & 0.026 & 0.004 & \textbf{0.003} & 0.027 & \textbf{0.003} & -34.50 & -22.56 & -33.74 & \textbf{-35.03} & -26.83 & \textbf{-35.03} \\
 & DeepAR & 0.012 & 0.065 & \textbf{0.006} & 0.114 & 0.009 & 0.114 & \textbf{-28.47} & -17.98 & -27.50 & -29.07 & -21.76 & -29.07 \\
 & GP-Copula & \textbf{0.004} & 0.042 & 0.016 & 0.088 & 0.026 & 0.088 & \textbf{-32.85} & -20.47 & -30.81 & -32.31 & -26.16 & -32.31 \\
 & NsDiff & \textbf{0.004} & 0.042 & 0.011 & 0.128 & 0.011 & 0.091 & \textbf{-28.63} & -21.37 & -25.40 & -28.80 & -21.38 & -28.37 \\
 & TACTiS-2 & \textbf{0.004} & 0.036 & 0.020 & 0.006 & 0.024 & 0.006 & -34.39 & -22.68 & -33.41 & \textbf{-34.39} & -27.13 & \textbf{-34.39} \\
 & TimeGrad & \textbf{0.009} & 0.026 & 0.473 & 0.845 & 0.838 & 0.845 & \textbf{-10.88} & -8.83 & -12.23 & -17.64 & -6.17 & -17.64 \\
 & TMDM & \textbf{0.009} & 0.026 & 0.034 & 0.374 & 0.246 & 0.364 & \textbf{-12.82} & -9.70 & -14.77 & -18.17 & -9.52 & -18.01 \\
 & TSFlow & \textbf{0.003} & 0.039 & 0.101 & 0.118 & 0.095 & 0.078 & \textbf{-31.84} & -20.47 & -32.35 & -34.26 & -26.32 & -33.38 \\
\midrule
\multirow{7}{*}{traffic} & DeepAR & \textbf{0.003} & 0.067 & 0.027 & 0.417 & 0.049 & 0.417 & \textbf{-1857.94} & -1533.90 & -1786.11 & -1110.90 & -1279.82 & -1110.90 \\
 & GP-Copula & \textbf{0.002} & 0.057 & \textbf{0.002} & 0.467 & 0.047 & 0.467 & \textbf{-1920.81} & -1373.97 & -1677.17 & -1047.90 & -1245.06 & -1047.90 \\
 & NsDiff & \textbf{0.006} & 0.028 & 0.034 & 0.553 & 0.026 & 0.553 & \textbf{-1537.94} & -968.88 & -1653.99 & -1090.52 & -1389.24 & -1090.52 \\
 & TACTiS-2 & \textbf{0.004} & 0.054 & 0.018 & 0.469 & 0.040 & 0.469 & \textbf{-1766.47} & -1485.68 & -1695.04 & -1075.69 & -1308.90 & -1075.69 \\
 & TimeGrad & \textbf{0.001} & 0.050 & 0.033 & 0.709 & 0.031 & 0.709 & \textbf{-1292.86} & -1205.61 & -1273.84 & -955.67 & -929.23 & -955.67 \\
 & TMDM & \textbf{0.003} & 0.043 & 0.027 & 0.486 & 0.026 & 0.486 & \textbf{-1414.25} & -629.93 & -1600.10 & -1086.10 & -1365.39 & -1086.10 \\
 & TSFlow & \textbf{0.002} & 0.035 & 0.031 & 0.471 & 0.039 & 0.471 & \textbf{-1866.32} & -1488.75 & -1755.73 & -1111.23 & -1445.21 & -1111.23 \\
\midrule
\multirow{8}{*}{weather} & CSDI & \textbf{0.001} & 0.015 & 0.022 & 0.065 & 0.007 & 0.035 & 49.05 & \textbf{8.05} & 8.75 & 31.59 & 111.77 & 33.00 \\
 & DeepAR & 0.008 & \textbf{0.007} & 0.020 & 0.137 & 0.029 & 0.078 & 60.59 & \textbf{37.36} & 42.74 & 51.66 & 121.18 & 54.18 \\
 & GP-Copula & \textbf{0.000} & 0.001 & 0.007 & 0.071 & 0.005 & 0.040 & 49.41 & \textbf{25.51} & 27.73 & 46.84 & 112.24 & 49.05 \\
 & NsDiff & \textbf{0.001} & 0.007 & 0.002 & 0.094 & 0.002 & 0.068 & 66.01 & \textbf{55.47} & 57.71 & 61.58 & 113.15 & 62.93 \\
 & TACTiS-2 & \textbf{0.001} & 0.003 & 0.011 & 0.060 & 0.011 & 0.036 & 50.75 & 10.75 & \textbf{9.63} & 31.75 & 112.39 & 33.28 \\
 & TimeGrad & 0.007 & \textbf{0.004} & 0.105 & 0.273 & 0.023 & 0.241 & 61.43 & \textbf{42.84} & 46.47 & 48.32 & 114.37 & 50.19 \\
 & TMDM & \textbf{0.001} & 0.004 & 0.003 & 0.107 & 0.005 & 0.057 & 55.34 & \textbf{37.29} & 38.28 & 45.55 & 109.12 & 47.81 \\
 & TSFlow & \textbf{0.001} & 0.178 & 0.021 & 0.211 & 0.006 & 0.165 & \textbf{64.55} & 68.04 & 80.36 & 45.53 & 114.72 & 46.73 \\
\end{longtable}

\endgroup

\section{Non-stationarity and Performance Gains Analysis}
\label{app:covdrift_performance}

In this appendix, we examine whether the performance gains of \textsc{Space} are associated with the severity of time-varying multivariate dependence structure. Our goal is to test whether datasets with stronger and more rapidly changing joint covariance geometry are precisely the settings in which modeling multivariate uncertainty structure is most beneficial.

\paragraph{Local covariance drift index.}
For each dataset, let $Y_t \in \mathbb{R}^d$ denote the standardized multivariate target at time $t$. We quantify local covariance instability by comparing adjacent rolling windows. For a window size $w$, define
\begin{equation}
W_t^{L} = \{Y_{t-w}, \dots, Y_{t-1}\},
\qquad
W_t^{R} = \{Y_t, \dots, Y_{t+w-1}\},
\end{equation}

with empirical covariance matrices $\Sigma_t^{L}$ and $\Sigma_t^{R}$. Let $\Sigma_{\mathrm{full}}$ denote the covariance matrix of the full standardized series, and define a whitening operator
\begin{equation}
A = (\Sigma_{\mathrm{full}} + \lambda I)^{-1/2},
\end{equation}
where $\lambda > 0$ is a small ridge term used only for numerical stability. We then compute the local covariance-drift score
\begin{equation}
\delta_t
=
\frac{\left\|A(\Sigma_t^{L} - \Sigma_t^{R})A\right\|_F}{\sqrt{d}}.
\end{equation}
This quantity is dimension-normalized and measures how much the local covariance structure changes from one temporal regime to the next, relative to the dataset's own global covariance geometry. To obtain a single dataset-level summary, we aggregate the largest drift episodes using the mean of the top decile of $\{\delta_t\}_t$:
\begin{equation}
\mathrm{LCI}
=
\operatorname{Top10Mean}\bigl(\{\delta_t\}_t\bigr),
\end{equation}
where larger values indicate stronger local covariance drift. In the plots below, datasets are ordered from left to right by decreasing $\mathrm{LCI}$.

\paragraph{Interpretation.}
The local covariance drift index is intended to capture a specific form of nonstationarity: instability in the joint second-order geometry of the target. This is directly relevant for \textsc{Space}, whose construction relies on adaptive multivariate covariance structure when forming calibrated uncertainty sets. If joint covariance geometry is both meaningful and time-varying, then methods that explicitly account for such geometry should benefit most.

\paragraph{Results and ranking.}
Using the local covariance drift index defined above, the datasets are ranked as follows:
\texttt{traffic} ($4.939$),
\texttt{weather} ($3.326$),
\texttt{electricity} ($3.108$),
\texttt{etth1} ($1.738$),
\texttt{exchange} ($1.487$),
\texttt{ettm2} ($1.075$),
and \texttt{ettm1} ($0.867$).
Thus, the left-hand side of Figure~\ref{fig:local_covdrift_all_baselines} corresponds to datasets with substantially stronger local covariance deformation over time, while the right-hand side contains the smaller and more stationary regimes. In particular, \texttt{ettm1} has the smallest local covariance drift index, consistent with its comparatively stable dependence structure.

\paragraph{Association with coverage-gap improvement.}
This ranking is informative for understanding where \textsc{Space} helps most. At the pair level, the local covariance drift index is positively associated with coverage-gap improvement, with Pearson correlation $r=0.357$ ($p=0.0080$) and Spearman correlation $\rho=0.281$ ($p=0.0396$). At the dataset level, where the sample size is necessarily small ($n=7$), the association remains positive but is less statistically stable, with Pearson $r=0.596$ ($p=0.158$) and Spearman $\rho=0.679$ ($p=0.0938$). Looking directly at the dataset means, the strongest gains appear in more difficult regimes such as \texttt{electricity}, whose local covariance drift index is $3.108$ and whose mean coverage-gap improvement is $0.0344$, and \texttt{traffic}, whose local covariance drift index is $4.939$ with mean improvement $0.0217$. By contrast, smaller and more stationary datasets exhibit much smaller gains; for example, \texttt{ettm1} has the lowest local covariance drift index ($0.867$) and also one of the smallest mean improvements ($0.0035$). Overall, the figure supports the interpretation that \textsc{Space} is most beneficial when the target exhibits stronger time-varying multivariate dependence geometry, while performance becomes much closer to the baselines in lower-drift settings.

\paragraph{SPACE versus all baselines.}
To make this trend visible at a dataset level, Figure~\ref{fig:local_covdrift_all_baselines} compares \textsc{Space} against the mean performance of all baseline methods, with datasets ordered by decreasing local covariance drift. The top panel reports mean coverage gap, and the bottom panel reports mean rolling gap. In both panels, points correspond to dataset means across dataset--forecaster pairs, and error bars show mean $\pm 2$ standard errors. The raw points in the background show the individual dataset--forecaster pairs. The same qualitative pattern appears in both panels: \textsc{Space} offers the clearest advantage in the high-drift regime, while performance differences narrow substantially in the low-drift regime.

\begin{figure}[t]
    \centering
    \includegraphics[width=\linewidth]{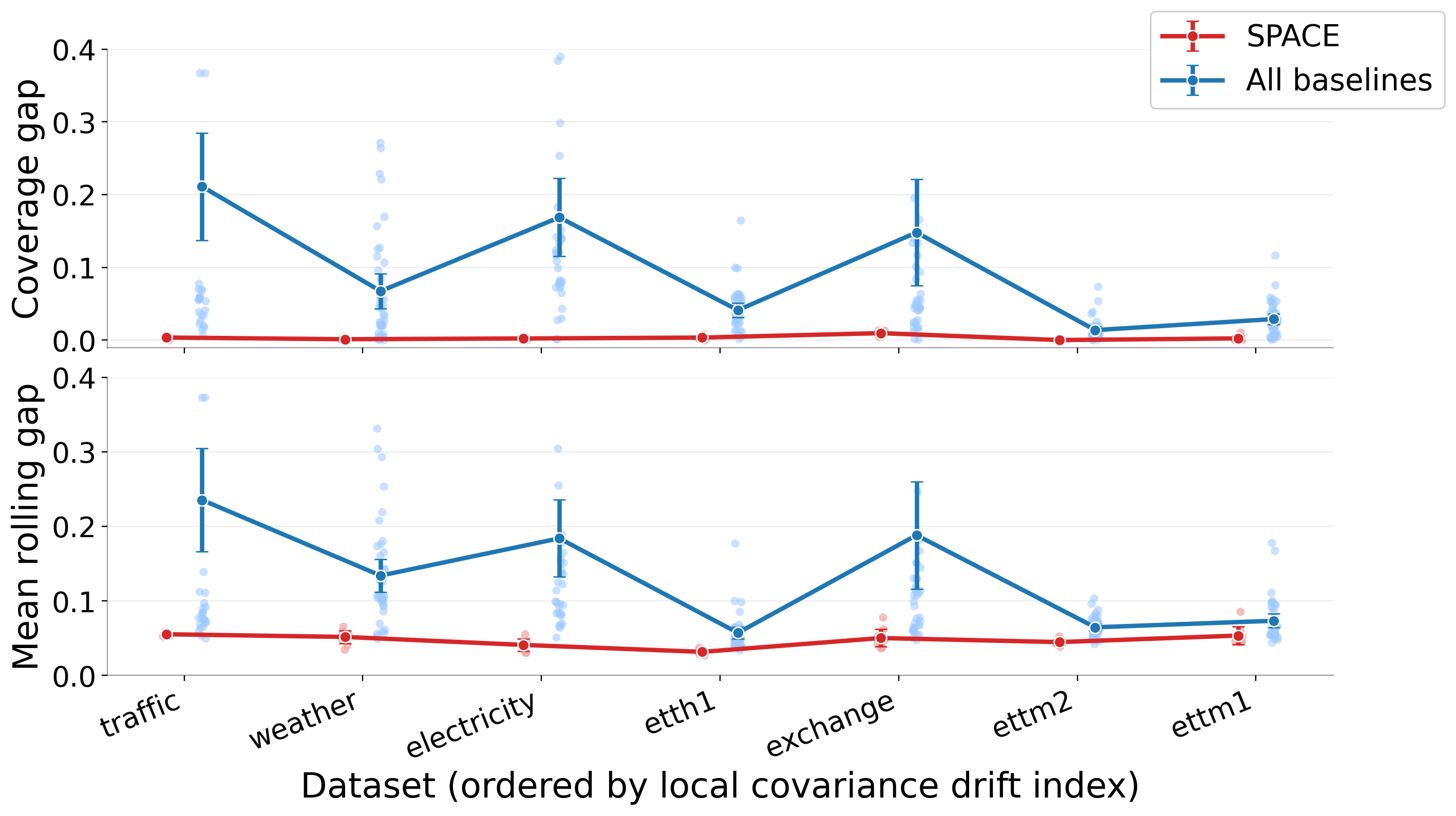}
    \caption{Performance of \textsc{Space} versus the mean over all baselines, with datasets ordered from left to right by decreasing local covariance drift index. Top: mean coverage gap. Bottom: mean rolling gap. Error bars are constructed based on two standard errors of the mean across the dataset--forecaster pairs. The left side of the figure corresponds to higher-dimensional and more nonstationary datasets with stronger local covariance drift, where the advantage of \textsc{Space} is largest; the gap narrows for smaller and more stationary datasets such as \texttt{ettm1}.}
    \label{fig:local_covdrift_all_baselines}
\end{figure}

Overall, this appendix analysis is consistent with the view that \textsc{Space} is especially effective in datasets whose multivariate dependence geometry is both strong and temporally unstable. In contrast, when the target is lower-dimensional and closer to stationary covariance structure, \textsc{Space} remains competitive but its advantage over simpler baselines naturally becomes smaller.

\section{Dynamic Window-Length Diagnostics}
\label{app:dynamicL}

In this appendix, we provide a visual diagnostic for the dynamic window length selected by \textsc{Space}, denoted by \(L_t\). The goal is to illustrate the intended mechanism behind the dynamic rule. When the underlying target process changes regime, the selected memory length is designed to contract so that stale calibration history is downweighted or discarded. When the process becomes more stable again, the selected memory length can recover, allowing the method to reuse longer and more stable calibration windows.

To avoid circularity, we detect regime changes directly from the original multivariate target stream \(y_t\), rather than from the selected \(L_t\) path or from \textsc{Space} diagnostics. Specifically, for each dataset--forecaster pair at the \(90\%\) target level, we align the valid test times, standardize the multivariate target stream, reduce it by PCA to a low-dimensional representation, and then apply the energy-divisive nonparametric multivariate change-point method of \citet{matteson2014nonparametric}. This provides a target-side notion of regime change that is external to the dynamic window-selection rule itself.

Figure~\ref{fig:dynamicL_examples} is constructed to visualize three aspects of the dynamic rule:
\begin{enumerate}[topsep=1pt, itemsep=1pt]
    \item \textbf{Contraction:} whether the selected window length becomes smaller around a detected change point;
    \item \textbf{Recovery:} whether the selected window length increases again after the immediate post-change region;
    \item \textbf{Long-term behavior:} whether the selected window length tends to grow as the detected regime becomes older and more stable.
\end{enumerate}

Each panel in Figure~\ref{fig:dynamicL_examples} contains three aligned subplots. The top subplot shows the selected dynamic window length \(L_t\) over valid test times. The middle subplot shows \emph{Target PC1}, the first principal component of the standardized multivariate target stream, which provides a one-dimensional summary of the dominant variation in the target process. The bottom subplot shows target-side change points detected by the energy-divisive procedure.

The intended visual pattern is as follows. Around a vertical change-point marker, a local downward reset in \(L_t\) indicates that the dynamic rule shortens its calibration memory when the target process changes. Over a more stable segment, \(L_t\) may increase again, indicating that the method can use a longer historical window when the local regime persists.

Taken together, these illustrative plots are consistent with the intended interpretation of the dynamic window-selection mechanism. The selected memory length is not static; rather, it changes over time, contracts around detected target-side disturbances, and in the clearer examples rebuilds during more stable segments. This behavior reflects the design motivation for dynamic \(L_t\): \textsc{Space} uses shorter memory when the recent past becomes unreliable and longer memory when the local regime appears stable enough to benefit from a broader calibration history.

\begin{figure}[t]
    \centering
    \begin{subfigure}[t]{0.49\textwidth}
        \centering
        \includegraphics[width=\linewidth]{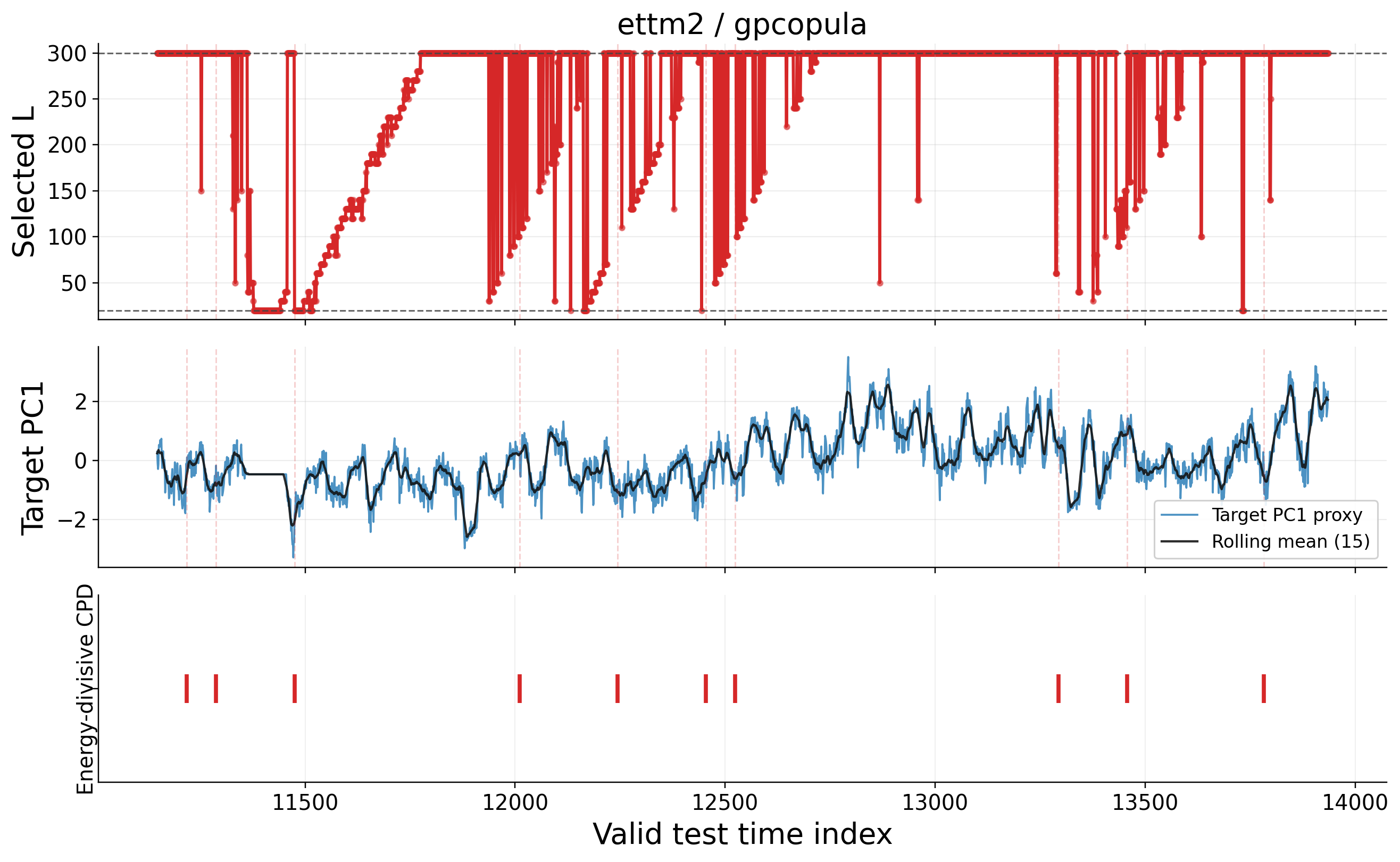}
        \caption{\texttt{ettm2/gpcopula}}
    \end{subfigure}
    \hfill
    \begin{subfigure}[t]{0.49\textwidth}
        \centering
        \includegraphics[width=\linewidth]{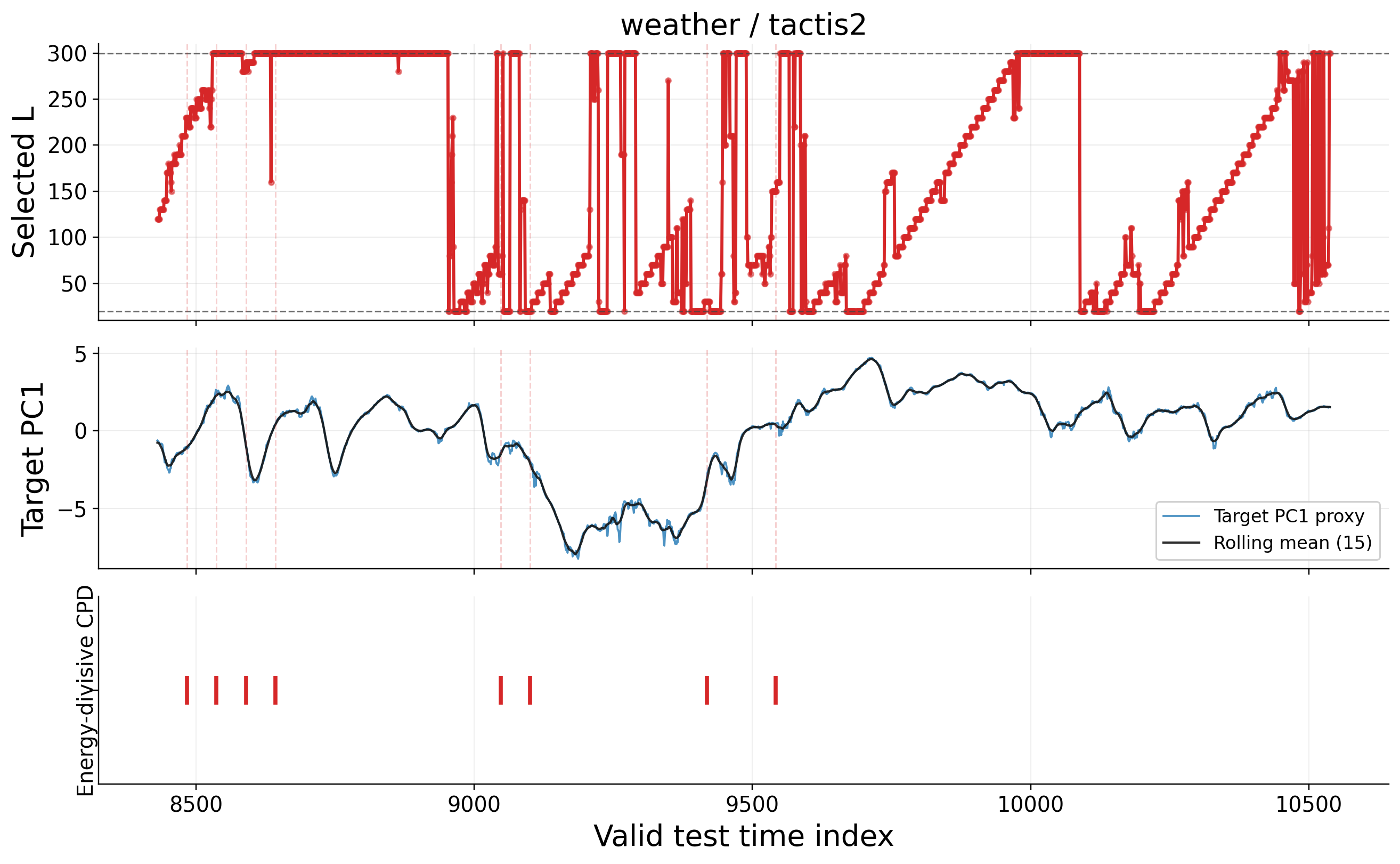}
        \caption{\texttt{weather/tactis2}}
    \end{subfigure}
    \caption{Illustrative dynamic-\(L\) overlays at the \(90\%\) target level using energy-divisive target-side change-point detection~\citep{matteson2014nonparametric}. In each plot, the top panel shows the selected window length \(L_t\), the middle panel shows the first principal component of the standardized valid-test target stream (Target PC1), and the bottom panel shows detected target-side change points. Across these examples, the selected dynamic memory length exhibits active interior adaptation and contraction-and-rebuilding behavior around detected target-side regime changes.}
    \label{fig:dynamicL_examples}
\end{figure}

\section{Implementation details of \textsc{Space}}
\label{app:space_impl}

\subsection{Metric definitions}
\label{app:metric_definitions}

For a fixed target miscoverage level $\alpha$, empirical joint coverage is defined as
\begin{equation}
    \widehat{\mathrm{Cov}}
    =
    \frac{1}{N}\sum_{t=1}^N
    \mathbb{I}\{y_t \in \mathcal{C}_t(\alpha)\}.
\end{equation}
The global coverage gap is the absolute deviation from the target coverage level,
\begin{equation}
    \mathrm{Gap}
    =
    \left|\widehat{\mathrm{Cov}} - (1-\alpha)\right|.
\end{equation}

For rolling diagnostics, let $w$ denote the sliding-window length and let $\mathcal{T}$ denote the set of valid rolling-window endpoints. The rolling joint coverage at endpoint $t \in \mathcal{T}$ is
\begin{equation}
    \widehat{\mathrm{RC}}_t
    =
    \frac{1}{w}\sum_{s=t-w+1}^{t}
    \mathbb{I}\{y_s \in \mathcal{C}_s(\alpha)\}.
\end{equation}
The corresponding rolling coverage gap is
\begin{equation}
    g_t
    =
    \left|\widehat{\mathrm{RC}}_t - (1-\alpha)\right|.
\end{equation}

We summarize these local deviations using three rolling metrics. The mean rolling gap is
\begin{equation}
    \mathrm{Mean\ Rolling\ Gap}
    =
    \frac{1}{|\mathcal{T}|}
    \sum_{t \in \mathcal{T}} g_t
    =
    \frac{1}{|\mathcal{T}|}
    \sum_{t \in \mathcal{T}}
    \left|\widehat{\mathrm{RC}}_t - (1-\alpha)\right|.
\end{equation}
The P90 rolling gap is the $90$th percentile of the rolling coverage gaps,
\begin{equation}
    \mathrm{P90\ Rolling\ Gap}
    =
    Q_{0.9}\!\left(\{g_t : t \in \mathcal{T}\}\right).
\end{equation}
The fraction of bad windows is the proportion of rolling windows whose gap exceeds a tolerance threshold $\delta$,
\begin{equation}
    \mathrm{Frac.\ Bad\ Windows}
    =
    \frac{1}{|\mathcal{T}|}
    \sum_{t \in \mathcal{T}}
    \mathbb{I}\{g_t > \delta\}.
\end{equation}
In all main experiments, we use $w=30$ and $\delta=0.10$.

For efficiency, we report the mean log-volume of the prediction regions,
\begin{equation}
    \textsc{LogVol}
    =
    \frac{1}{N}\sum_{t=1}^N
    \log \mathrm{Vol}\!\left(\mathcal{C}_t(\alpha)\right).
\end{equation}
Lower values are better for coverage gap, rolling diagnostics, and \textsc{LogVol}. The size metric should be interpreted together with calibration, since small sets are meaningful only when coverage is close to the target.

\subsection{Compute resources}
\label{app:compute_resources}

All wrapper experiments were run on a single local workstation equipped with a 13th Gen Intel(R) Core(TM) i9-13980HX CPU, an NVIDIA GeForce RTX 4080 Laptop GPU and 32GB system memory. The reported benchmark evaluates all dataset--forecaster pairs and all target levels included in the full result table in Section~\ref{app:full_result_tables}. Since \textsc{Space} is a post-hoc wrapper, the reported time refers to wrapper execution and evaluation on saved predictive sample streams, excluding the original training of the forecasting backbones. The complete \textsc{Space} run took 1h19m38s. MultiDimSPCI was the most computationally intensive among the baselines, requiring 14h58m33s under the same benchmark scope. The remaining conformal baselines were substantially lighter and completed in comparatively small additional time. Additional exploratory runs and preliminary tuning experiments were performed during development, but the above times describe the compute required to reproduce the reported final benchmark from the saved sample streams.

\subsection{Experimental configuration}
Unless stated otherwise, all forecasters use context window $96$, and one-step-ahead prediction (\texttt{pred\_len}$=1$) to simulate one-step-ahead forecasting. The saved predictive sample streams have $M=100$ forecast samples per time point. Note that, in principle, increasing the number of forecast samples reduces sample-geometry error. In the main benchmark, \textsc{Space} is evaluated under a single fixed configuration. Specifically, the candidate calibration lengths are
\begin{equation}
L \in \{20,30,\dots,300\},
\end{equation}
so that $L_{\min}=20$, $L_{\max}=300$, and the backward extension step size is $h=10$. The probe-block length is set to $p=20$, and the adaptive conformal inference update uses step size $\eta=0.01$.

For numerical robustness, the benchmark implementation uses a shrinkage-stabilized version of the sample-derived covariance estimator, with identity shrinkage parameter $\lambda_{\text{shrinkage}}=0.30$. This shrinkage choice is an empirical implementation device introduced to improve stability of the local covariance estimate in practice, especially when the effective sample size is limited or the sample cloud is ill-conditioned. It is not a separate conceptual component of \textsc{Space}, and it is not required by the theoretical results in Section~\ref{sec:theory}. Our theory is stated at the level of the sample-derived local geometry and its estimation error, so the guarantees do not rely on this particular shrinkage form or on the specific value of $\rho$.

For dynamic same-regime calibration window selection, \textsc{Space} uses a DKW-style KS acceptance threshold to instantiate the same-regime test described in Section~\ref{subsec:dynamic_window}. At time $t$, the threshold is
\begin{equation}
\tau_t(c_\tau)
=
c_\tau
\sqrt{
\frac{
\log\!\left(\frac{2K_tJ}{\delta_\tau}\right)
}{
2p
}
},
\end{equation}
where $c_\tau=2.0$ is a multiplicative tuning constant, $K_t$ is the number of feasible candidate window lengths at time $t$, $J=3$ is the number of diagnostics in the test bundle, $\delta_\tau = 0.05$, and $p=20$ is the probe-block length. The factor $2K_tJ$ accounts for the two block comparisons, the feasible candidate lengths, and the diagnostic bundle. This threshold is the empirical implementation of the KS-style same-regime test; the theoretical analysis abstracts its behavior through Assumption~(A2), which controls false rejection on clean blocks and false acceptance on contaminated extensions.

\section{Dataset Details}
\label{app:asset_licenses}
We use the following seven publicly available datasets for academic research and evaluation purposes. The details of the datasets are listed in Table~\ref{tab:dataset_summary}.

\begin{table}[t]
\caption{Dataset summary for the saved predictive sample streams used in wrapper evaluation. Here $d$ is the target dimension, $T$ is the saved stream length, and Test denotes the \textsc{Space}/baseline wrapper test split.}
\label{tab:dataset_summary}
\centering
\small
\setlength{\tabcolsep}{5pt}
\begin{tabular}{lcccc}
\toprule
Dataset & Freq. & $d$ & $T$ & Test \\
\midrule
electricity & H & 321 & 5261 & 1053 \\
ETTh1 & H & 7 & 3484 & 697 \\
ETTm\{1,2\} & 15 min & 7 & 13936 & 2788 \\
exchange & D & 8 & 1518 & 304 \\
traffic & H & 862 & 3509 & 702 \\
weather & 10 min & 21 & 10539 & 2108 \\
\bottomrule
\end{tabular}
\end{table}

\begin{enumerate}
    \item \texttt{Electricity}. 
    The dataset records electricity consumption for 321 customers and is used as a standard
    multivariate forecasting benchmark.
    URL: \url{https://archive.ics.uci.edu/dataset/321/electricityloaddiagrams20112014}.

    \item \texttt{ETTh1, ETTm1, and ETTm2}. 
    The datasets contain electricity transformer variables, including load and oil temperature,
    recorded at hourly or 15-minute resolution~\cite{zhou2021informer}.
    URL: \url{https://github.com/zhouhaoyi/ETDataset}.

    \item \texttt{Exchange}. 
    The dataset records daily exchange rates of eight countries and is commonly used in
    multivariate time-series forecasting benchmarks~\cite{lai2018modeling}.
    URL: \url{https://github.com/laiguokun/multivariate-time-series-data}.

    \item \texttt{Traffic}. 
    The dataset contains hourly road occupancy rates measured by sensors on San Francisco Bay
    Area freeways and is commonly distributed through multivariate time-series benchmark
    collections~\cite{lai2018modeling}. 
    Raw data are associated with the Caltrans Performance Measurement System (PeMS), which
    requires an account for access.
    URL: \url{https://pems.dot.ca.gov/}; benchmark mirror:
    \url{https://github.com/laiguokun/multivariate-time-series-data}.

    \item \texttt{Weather}. 
    The dataset is commonly distributed through the Autoformer/Time-Series-Library benchmark
    collection~\cite{wu2021autoformer}.
    URL: \url{https://github.com/thuml/Autoformer}.
\end{enumerate}

\end{document}